\documentclass[journal]{IEEEtran}
\ifCLASSINFOpdf
\else
\fi
\usepackage[hidelinks]{hyperref}   % to hide link borders
\usepackage{hyperref}
\usepackage{subcaption}
\usepackage{epstopdf}
\usepackage{graphicx}
\usepackage{amsfonts}
\usepackage{enumerate}
\usepackage{varioref}
\usepackage{longtable}
\usepackage{graphicx}
\usepackage{amssymb}
\usepackage{chngpage}
\usepackage{multirow}
\usepackage{multirow}
\usepackage{amsthm}
\usepackage{amsmath}
\usepackage{caption}
\usepackage{nccmath}  % to left align an equation
\usepackage[table]{xcolor}  % for color in table
\usepackage{arydshln}  % dash line in table
\usepackage{algorithm}% http://ctan.org/pkg/algorithms               ----> % for algorithm (pseudo-code)
\usepackage[noend]{algpseudocode}% http://ctan.org/pkg/algorithmicx 		 ----> % for algorithm (pseudo-code)

\usepackage{threeparttable}
\definecolor{bananamania}{rgb}{0.98, 0.91, 0.71}

\usepackage{mathtools}  %%% https://tex.stackexchange.com/questions/42271/floor-and-ceiling-functions
\usepackage{amsthm} % for Lemma and theorem
\newtheorem{theorem}{Theorem}

\begin{document}

\title{Pontogammarus Maeoticus Swarm Optimization: \\A Metaheuristic Optimization Algorithm}

\author{Benyamin Ghojogh,
		Saeed Sharifian\\

\thanks{Benyamin Ghojogh's e-mail: \href{mailto:ghojogh_benyamin@ee.sharif.edu}{ghojogh\_benyamin@ee.sharif.edu}}
\thanks{Saeed Sharifian's e-mail: \href{mailto:sharifian_s@aut.ac.ir}{sharifian\_s@aut.ac.ir}}
\thanks{Both authors are with Department of Electrical Engineering, Amirkabir University of Technology, Tehran, Iran. This article was written in fall 2015 and revised in fall 2017.}}

%\markboth{To be Submitted to IEEE Transactions on Systems, Man, and Cybernetics - Part C: Applications and Reviews}%
%{Shell \MakeLowercase{\textit{et al.}}: Bare Demo of IEEEtran.cls for IEEE Journals}

\maketitle

\begin{abstract}
Nowadays, metaheuristic optimization algorithms are used to find the global optima in difficult search spaces. Pontogammarus Maeoticus Swarm Optimization (PMSO) is a metaheuristic algorithm imitating aquatic nature and foraging behavior. Pontogammarus Maeoticus, also called Gammarus in short, is a tiny creature found mostly in coast of Caspian Sea in Iran. 
In this algorithm, global optima is modeled as sea edge (coast) to which Gammarus creatures are willing to move in order to rest from sea waves and forage in sand. Sea waves satisfy exploration and foraging models exploitation. The strength of sea wave is determined according to distance of Gammarus from sea edge. The angles of waves applied on several particles are set randomly helping algorithm not be stuck in local bests. 
Meanwhile, the neighborhood of particles change adaptively resulting in more efficient progress in searching.
The proposed algorithm, although is applicable on any optimization problem, is experimented for partially shaded solar PV array. Experiments on CEC05 benchmarks, as well as solar PV array, show the effectiveness of this optimization algorithm.
\end{abstract}

% Note that keywords are not normally used for peerreview papers.

\IEEEpeerreviewmaketitle

%%%%%%%%%%%%%%%%%%%%%%%%%%%%%%%%%%%%%%%%%
\section{Introduction}\label{section_introduction}
\IEEEPARstart{M}{etaheuristic} algorithms for optimization purposes have been gotten attention as much as the classic mathematical solutions. The reason can be seeked out in this fact that metaheuristic methods are based on soft computing and can perform better in complicated problems, especially in problems with lots of local solutions. On the other hand, mathematical optimization algorithms are usually harder to use since they may be stuck in local solutions \cite{talbi2009metaheuristics,yang2010nature}. Metaheuristic algorithms are actually one of the branches of artificial intelligence and are mostly used to optimize a cost function as well as classical algorithms \cite{engelbrecht2007computational}.

There are lots of metaheuristic algorithms proposed for different applications. Several well-known ones are introduced here. Genetic Algorithm (GA), proposed by Holland \cite{holland1975adaptation} and Goldberg \cite{goldberg1989genetic} is a genetic-inspired algorithm dealing with mutation and cross-over of chromosomes simulating the possible solutions. Genetic programming, proposed by Koza \cite{koza1992genetic}, is also a chromosome-based algorithm with mutations and cross-overs. However, it is mostly used for solving tree-based problems. This algorithm was first introduced to extend artificial intelligence to programming goals. It is usually used to model the problem with a mathematical function. Evolutionary programming, proposed by Fogel \emph{et. al.} \cite{fogel1966artificial}, De Jong \cite{de1975analysis} and Koza \cite{koza1990genetic}, is a chromosome-based algorithm with a mutation operator and a strategy parameter. Tabu search, proposed by Glover \cite{glover1977heuristics}, prepares a tabu list to list illegal paths of search in it. Simulated annealing, proposed by Kirkpatrick \emph{et. al.} \cite{kirkpatrick1983optimization}, simulates the annealing process to solve an optimization problem. Temperature is a parameter of this algorithm, which decreases by going forward in the iterations. Particle Swarm Optimization (PSO), proposed by Kennedy and Eberhart \cite{kennedy2011particle}, is inspired by the behavior of bird or fish swarms and simulates the migration of swarm particles. The particles search the landscape and the global and local bests affect the direction and speed of particles' movements.

Lots of metaheuristic algorithms have simulated the behavior of aquatic creatures. As an example, krill herd algorithm, proposed by Gandomi and Alavi \cite{gandomi2012krill}, can be mentioned in which three things affect the movement of the krill individuals: (I) Induced Movement, (II) Foraging and (III) Random Diffusion. There are some other papers improving this algorithm, such as papers \cite{guo2014new} and \cite{saremi2014chaotic}. Artificial algae algorithm \cite{uymaz2015artificial} is another example of aquatic-inspired algorithms. 
In this article, a new metaheuristic optimization algorithm, Pontogammarus Maeoticus Swarm Optimization (PMSO), is proposed which is also inspired by an aquatic creature named Pontogammarus Maeoticus or shortly Gammarus.

In PMSO algorithm, the behavior of Pontogammarus Maeoticus swarm is simulated as a metaheuristic optimization algorithm. Pontogammarus Maeoticus, or Gammarus, is a tiny sea creature especially found in sea edges (coasts). In this algorithm, local search exploitation is satisfied by modeling foraging of Gammarus individuals in their neighborhood. They are willing to reach the sea edge to settle in there and be in peace from the sea waves. Hence, global optima is modeled as the sea edge and the particles search for it to reach. In order to have exploration on different parts of search space, every Gammarus is shaken using sea wave whose strength is related to distance of the Gammarus from the best answer found so far. 
The sea waves applied on Gammarus individuals which are close to sea edge, are toward the sea edge to help better exploitation around global best found so far. On the other hand, the particles far from sea edge are moved randomly in order to have better exploration in unseen parts of search space and escaping local optimums.
The neighborhoods of Gammarus creatures are also subject to change adaptively resulting in better progress in searching for global optima.

As an example of its application, the proposed PMSO algorithm is also used for the problem of partial shading in solar arrays in order to maximize the output power.
Solar PV arrays consist of several PhotoVoltaic (PV) cells which are connected in different configurations to each other. The simplest structure of solar PV array is TCT configuration. This configuration, however, cannot cancel the effect of partial shadow on the array and therefore, the output power is not perfect. In order to roughly cancel the effect of shadow as much as possible, other configurations have been proposed during the years. For instance, the Su Do Ku structure is proposed based on the Su Do Ku puzzle arrangement \cite{woyte2003partial}. However, Su Do Ku configuration has several drawbacks \cite{deshkar2015solar}. Hence, recently, researchers have worked on solar PV array problem to enhance the output power. Although there are lots of works in non-heuristic areas for this problem such as \cite{bidram2012control,rakesh2016performance}, several researches are using metaheuristic optimization for that. Paper \cite{deshkar2015solar} can be mentioned as an example of the recent works on sollar array which utilize metaheuristic algorithms. In \cite{deshkar2015solar}, Genetic algorithm is used in order to improve the performance of solar PV array. 

The remainder of paper is organized as follows.  
Section \ref{Section_Intoducing_Gammarus} introduces Gammarus creature in a non-mathematical perspective. Different parts of PMSO optimization algorithm and its nature inspirations are explained in detail in Section \ref{Section_Algorithm}. 
In Section \ref{Section_Analysis}, convergence, time complexity, and space complexity of the proposed algorithm as well as a sample scenario of the algorithm are analyzed.
Experimental results on standard optimization benchmarks are reported in section \ref{Section_Experiments}.
Using PMSO algorithm for partially shaded solar PV array is proposed is Section \ref{Section_Solar}. 
Finally, Section \ref{Section_Conclusion} concludes the paper.

%%%%%%%%%%%%%%%%%%%%%%%%%%%%%%%%%%%%%%%%%
\section{Introducing Gammarus}\label{Section_Intoducing_Gammarus}
A special type of hard-shell creature (Malacostraca) is divided into two categories: (I) Amphipoda and (II) Isopoda \cite{vernberg1983biology}. Amphipoda is mostly found in sea beds, fresh waters, salty waters, wells, oceans, lakes, and beals. It usually exists in the sediments, on the rocks of seas, or in the sea edges \cite{pennak1953fresh}.
 
Gammarus is placed in Amphipoda category. Several pictures of Gammarus are illustrated in figure \ref{fig_gammarus}. Various types of Gammarus exist, and one of the well-known types is the one that is found in southern coast of Caspian sea in Iran. There are four types of Gammarus found in Caspian coast of Iran: (I) Pontogammarus Maeoticus, (II) Pontogammarus borcea, (III) Obesogammarus crassus, and (IV) Gammarus aequicauda \cite{shamsaei2009effects}.

\begin{figure}[!t]
\centering
\begin{subfigure}[b]{0.21\textwidth}
\centering
\includegraphics[width=1.5in]{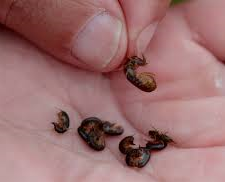} 
\caption{}
\end{subfigure}
\begin{subfigure}[b]{0.21\textwidth}
\centering
\includegraphics[width=1.5in]{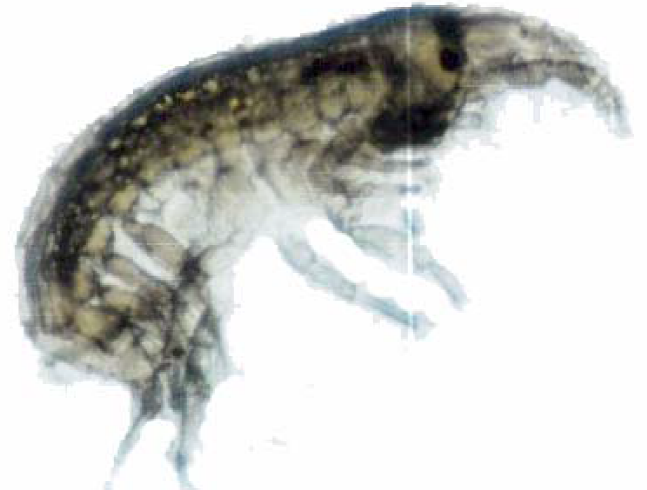}
\caption{}
\end{subfigure}
\caption{(a) Gammarus creature, (b) Gammarus body shape (credit of image: \cite{Abedian2003Analyzing}).}
\label{fig_gammarus}
\end{figure}

The material of Gammarus body consists of 48.4\% protein and 5.33\% lipid \cite{Choubert1995}. Gammarus usually reproduces by spawning on around September. Their percentage of protein and lipid increases and decreases right before and after spawning \cite{shamsaei2009effects}.
Gammarus has an average length of 12 millimeters and an average thickness of three millimeters \cite{ghareyazie2012studying,Seif2004Chemical}. The food of this creature is mostly organic materials and dead bodies of creatures. Gammarus has a very important role in food chain of aquatics such as fish. Its dried body is used as the food of aquarium fish \cite{Seif2004Chemical}. 
Gammarus has two antennas and several legs named gnathopods, pereopods, peleopods and uropods. Its shape of body is convex \cite{Azadkar2014Study}. See figure \ref{fig_body} for better visualization of its body parts.

\begin{figure}[!t]
\centering
\includegraphics[width=3in]{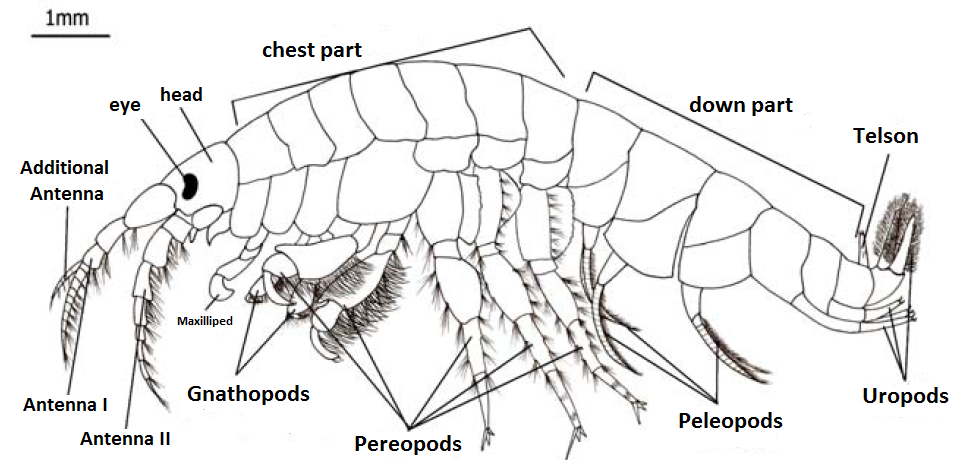}
\caption{Different parts of Gammarus body (image taken from \cite{Azadkar2014Study} with some modifications).}
\label{fig_body}
\end{figure}

\begin{figure*}[!t]
\centering
\includegraphics[width=6.4in]{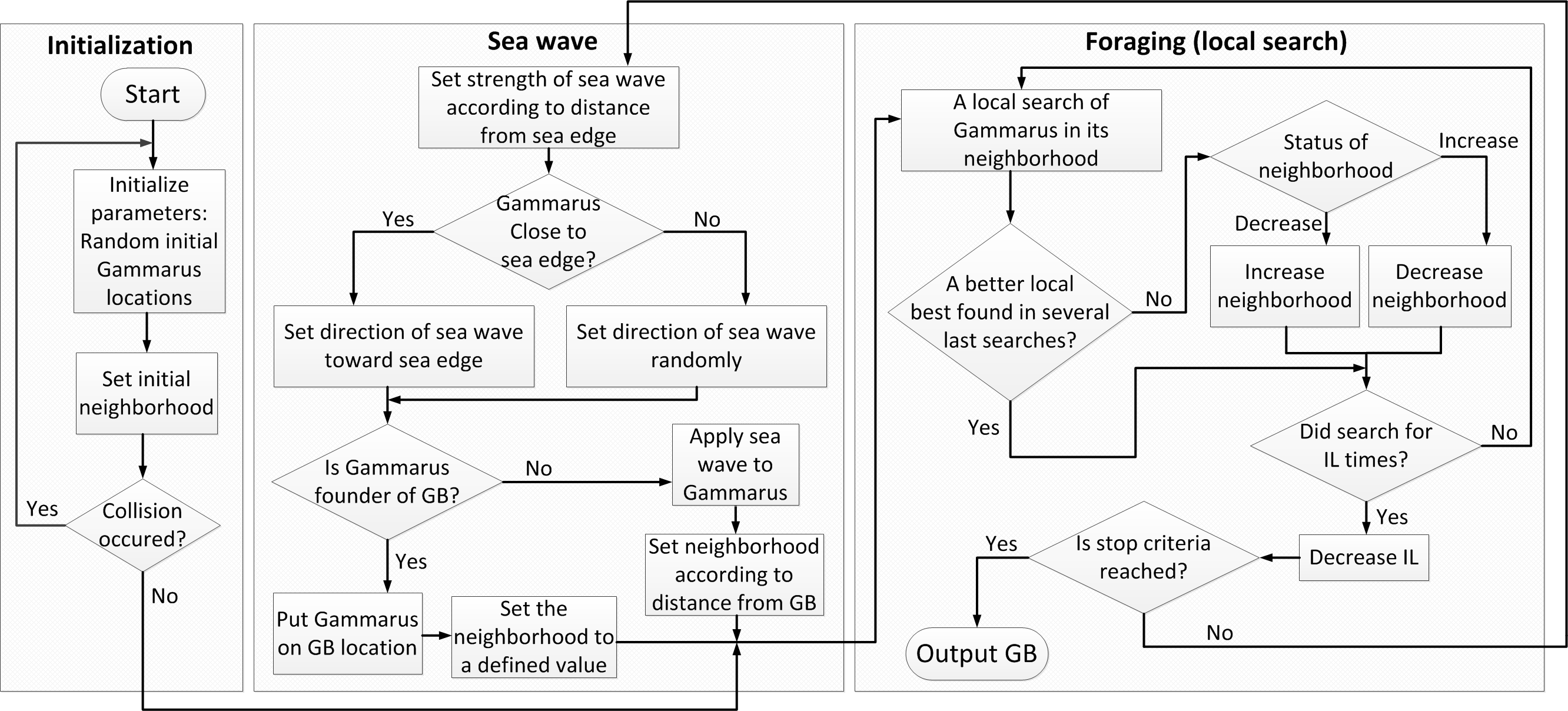}
\caption{Flow chart of PMSO algorithm.}
\label{fig_flowchart}
\end{figure*}

%%%%%%%%%%%%%%%%%%%%% table:
\begin{table}[!t]
%\begin{minipage}{\textwidth}
\renewcommand{\arraystretch}{1.3}  %%% each row size
\caption{Symbols used in PMSO algorithm}
\label{table_symbols}
\centering
\scalebox{0.8}{    %%% --> for resizing tables
\begin{tabular}{c l}
\hline
\hline
\multicolumn{1}{c}{\textbf{Symbol}}   & \multicolumn{1}{c}{\textbf{Meaning}} \\
\hline
$U(\alpha, \beta)$ & Uniform random number in the range of $[\alpha, \beta]$\\
\hline
NG & Number of Gammarus individuals\\
\hline
$\text{G}_i$ & $i^{th}$ Gammarus location in landscape\\
\hline
$X$ & Location in landscape (determined by vector $X$ from origin)\\
\hline
$D$ & Number of dimensions of landscape (search space)\\
\hline
$X(d)$ & Dimension $d$ of location $X$ in landscape\\
\hline
IG & Number of global iterations\\
\hline
IL & Number of iterations of local search\\
\hline
$\text{N}_i$ & Neighborhood of  $i^{th}$ Gammarus\\
\hline
$\text{S}_\text{IL}$ & Step of decrementing number of iterations of local search\\
\hline
$\text{L}_\text{IL}$ & Lower bound of decrementing number of iterations of local search\\
\hline
$\text{S}_\text{N}$ & Step of incrementing/decrementing neighborhood\\
\hline
$\text{L}_\text{N}$ & Lower bound of decrementing neighborhood\\
\hline
$\text{U}_\text{N}$ & Upper bound of incrementing neighborhood\\
\hline
$\text{W}_i$ & The wave vector affecting the $i^{th}$ Gammarus\\
\hline
$\Theta_i$ & Angle vector of wave, affecting $i^{th}$ Gammarus\\
\hline
GB & Global best, found so far\\
\hline
$\text{LB}_i$ & Local best in each global iteration, for $i^{th}$ Gammarus\\
\hline
B & Length of queue buffer\\
\hline
F & Fraction of distance from GB to be considered as neighborhood\\
\hline
$\text{St}_i$ & Status of $i^{th}$ Gammarus for neighborhood adaptation\\
\hline
$\text{N}_{\text{GB}}$ & Initial neighborhood of founder of GB at the first of every global iteration\\
\hline
\hline
\end{tabular}
}
%\end{minipage}
\end{table}

%%%%%%%%%%%%%%%%%%%%%%%%%%%%%%%%%%%%%%%%%
\section{PMSO Algorithm}\label{Section_Algorithm}

In PMSO algorithm, several Gammarus individuals are used as particles in a swarm searching for global best optima. 
The flowchart of PMSO algorithm is shown in figure \ref{fig_flowchart}. In addition, the pseudo-code of PMSO algorithm can be seen in Algorithm \ref{the_algorithm}. 
In the following, the details and different steps of this algorithm are explained, and the natural inspirations are addressed as well.
Symbols and notations used in this algorithm are summarized in table \ref{table_symbols}.

\subsection{Initialization}

\subsubsection{Random exploration}

As seen in figure \ref{fig_flowchart} and lines \ref{algorithm1_start_initialization}-\ref{algorithm1_end_initialization} of algorithm, Gammarus individuals are primarily explored randomly in the landscape. Random initialization of Gammarus individuals can be formulated as,
\begin{align}
G_i(d) & = U(-l_d,l_d), \label{equation1} \\
& -l_d \leq \text{landscape}(d) \leq l_d, \quad i = \{1,\dots,\text{NG}\}, \nonumber
\end{align}
where $G_i(d)$ is the location of Gammarus in dimension $d$ of landscape ($\text{landscape}(d)$), $[-l_d,l_d]$ is the bound of search space in dimension $d$, and NG denotes population of Gammarus swarm. 

\begin{algorithm}[!t]
\caption{PMSO Algorithm}\label{the_algorithm}
\begin{algorithmic}[1]
\State \textbf{Initialize} NG, IL, IG, initial $\text{N}_i$, $\text{S}_\text{IL}$, $\text{L}_\text{IL}$, $\text{S}_\text{N}$, $\text{L}_\text{N}$, $\text{U}_\text{N}$, B, F, and $\text{N}_{\text{GB}}$
\For{$i$ = 1 to NG}   \label{algorithm1_start_initialization}
	   \While{collision occurred}\label{algorithm1_start_collision}
		\State $\text{G}_i \gets U(-l, l)$
	\EndWhile\label{algorithm1_end_collision}
\EndFor  \label{algorithm1_end_initialization}
\While{stop criterion is not reached} \label{algorithm1_globalIteration}
	\For{$i$ = 1 to NG}
		\If{it is not first iteration} 
			\State $|\text{W}_i| \gets U(0,1) \times |\text{GB} - \text{G}_i|$  \label{algorithm1_waveStrength}
		    \If{$\text{G}_i$ is close to sea edge}
		    	\State $\Theta_i \gets $ equation (\ref{equation_angleTowardGlobalBest})
		    \Else
		    	\State $\Theta_i \gets $ equation (\ref{equation_angleRandom})
		    \EndIf
		    \State $\text{W}_i \gets $ equations (\ref{equation_waveVectorComponents}) and (\ref{equation_waveVector})
		    \If{this Gammarus $i$ is founder of GB} \label{algorithm1_start_updateLocation}
		    	\State $\text{G}_i \gets $ GB
		    \Else
				\State $\text{G}_i \gets \text{G}_i + \text{W}_i$
			\EndIf  \label{algorithm1_end_updateLocation}
		\EndIf
		\State $\text{N}_i \gets \text{F} \times |\text{GB} - \text{G}_i|$ \label{algorithm1_initialNeighborhood}
		\For{$j$ = 1 to IL} \label{algorithm1_localIteration}
			\State $\text{St}_i \gets $ NoChange \label{algorithm1_status}
			\State buffer $\gets$ Local Search \label{algorithm1_bufferLocalSearch}
			\If{$j \geq $ B} \label{algorithm1_start_changeNeighborhoodStatus}
				\If{$\text{LB}_i$ found in latest B local searches}
					\If{$\text{St}_i$ is NoChange or Decrease}
						\State $\text{St}_i \gets$ Increase
						\State $\text{N}_i \gets$ max($\text{N}_i + \text{S}_\text{N}, \text{U}_\text{N}$) 
					\ElsIf{$\text{St}_i $ is Increase}
						\State $\text{St}_i \gets$ Decrease
						\State $\text{N}_i \gets$ min($\text{N}_i - \text{S}_\text{N}, \text{L}_\text{N}$) 
					\EndIf
				\EndIf
			\EndIf  \label{algorithm1_end_changeNeighborhoodStatus}
		\EndFor
		\If{$\text{LB}_i$ is better than GB} \label{algorithm1_start_updateGB}
			\State GB $\gets \text{LB}_i$
		\EndIf  \label{algorithm1_end_updateGB}
	\EndFor
	\State $\text{IL} \gets$ max($\text{IL} - \text{S}_\text{IL}, \text{L}_\text{IL}$) \label{algorithm1_changeIL}
\EndWhile
\end{algorithmic}
\end{algorithm}

\begin{figure*}[!t]
\centering
\includegraphics[width=5in]{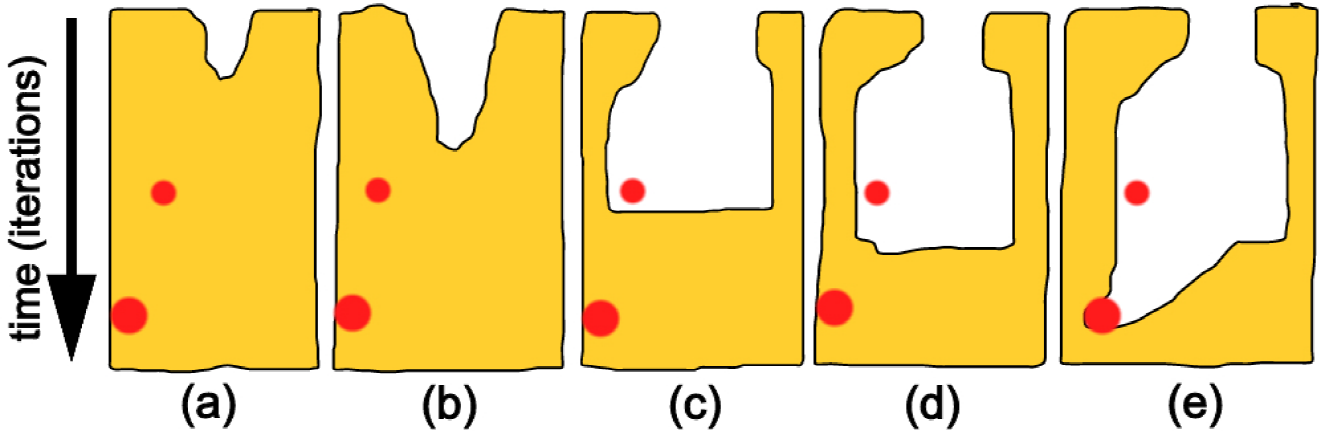}
\caption{Local search by Gammarus for the sake of foraging. The red dots are the nutrients (local bests).}
\label{fig_forage}
\end{figure*}

\subsubsection{Collision check}

After locating gammarus individuals randomly in landscape, collisions of them are checked and if occurs, their locations are set another time. For better time efficiency, this collision check is performed after locating every Gammarus individual rather than checking after locating all particles and re-locating all particles in case of collision (see lines \ref{algorithm1_start_collision}-\ref{algorithm1_end_collision} in algorithm). For having a measure of collision, neighborhood of each Gammarus is defined as a distance from it which can be a surrounding hyper-sphere or hyper-cube in $d$-dimensional space. Distance of locations $X_1$ and $X_2$ in landscape can be calculated using Euclidean distance,
\begin{equation}
\label{equation2}
\text{distance} (X_1, X_2) = |X_1 - X_2| = \sqrt{\sum_{d=1}^D \Big(X_1(d) - X_2(d) \Big)^2}.
\end{equation}
The initial neighborhood of all Gammarus particles are set equivalently to a number which is a hyper-parameter and should be set according to problem.

The benefit of this collision check is to have better initial exploration in landscape for searching different parts as much as possible.
Moreover, the avoidance of collision is inspired by the fact that Gammarus individuals do not like to get very close to each other. We have examined it ourselves with a container of water including several Gammarus creatures being in sand. When we shook the container, we observed that each Gammarus came out of the sand and after we put down the container, they tried to find a place in the sand to rest. However, whenever they landed to the place of another Gammarus, they went up again and tried another time to find a new place in the sand.

\subsection{Foraging}

\subsubsection{Natural inspiration}

As the next step, for the sake of exploitation, each Gammarus searches locally its neighborhood for several iterations. The local search of each Gammarus in its neighborhood is inspired by its foraging for which it searches around in the sand it has landed on. Figure \ref{fig_forage} depicts the local search for foraging in sand.  
As can be seen in this figure, there are two pieces of nutrients in the sand which the Gammarus is looking for to forage. At first, the Gammarus falls on the region to search (figure \ref{fig_forage}(a)). After the Gammarus searches within its neighborhood for several iterations (figure \ref{fig_forage}(b)), it finds out that not any nutrient has been existed in that region, so it increases its neighborhood for search and finds a nutrient (see figure \ref{fig_forage}(c)). The nutrients are modeled by local bests found by Gammarus individuals. The Gammarus keeps searching with the same increased neighborhood for several iterations; however, it does not find any nutrient again (figure \ref{fig_forage}(d)). So, it guesses that it might be better to search in a specific direction and it might luckily find a nutrient with less effort (figure \ref{fig_forage}(e)). 

\subsubsection{Exploitation}

Notice that in this algorithm, two types of iterations exist: (I) global iteration that is the iteration on the whole algorithm (line \ref{algorithm1_globalIteration} of algorithm), and (II) local iteration that is the iteration of local search of each Gammarus in its neighborhood (line \ref{algorithm1_localIteration}). 
The number of local iterations in every global iteration is denoted as IL. 
The aim of local searches performed by Gammarus individuals is to perform exploitation in the locations of landscape on which Gammarus particles are located in a global iteration.
The local search of particle is inspired by foraging of Gammarus which looks for nutrient in the sand on which it has landed for a moment before facing a sea wave. 

\subsubsection{Adaptive neighborhood}\label{section_adaptive_neighborhood}

Inspired by the explained natural foraging behavior, the neighborhood of Gammarus individuals are set to be adaptive. 
A status label, denoted as $\text{St}_i$, is defined for each Gammarus. This status is initially set to NoChange at the local search iterations within every global iteration (see line \ref{algorithm1_status} in algorithm), and the Gammarus searches within its initial neighborhood. 
Note that this initial neighborhood is subject to change in every global iteration, and setting the initial neighborhood for each Gammarus is addressed in the next sections. 
A queue buffer with length B is defined per each Gammarus in which the results of latest B local searches of that Gammarus are stored (line \ref{algorithm1_bufferLocalSearch} of algorithm). This buffer is emptied at the first of each global iteration. If the local best of Gammarus in that global iteration, which is $\text{LB}_i$, is found in the latest B searches stored in buffer, it means the procedure of searching is progressive and Gammarus does not change its status $\text{St}_i$ and searches within the previously set neighborhood. Otherwise, it changes its status to Decrease or Increase if its previous status is Increase or Decrease, respectively (see lines \ref{algorithm1_start_changeNeighborhoodStatus}-\ref{algorithm1_end_changeNeighborhoodStatus} of algorithm). Accordingly, the neighborhood of Gammarus individual is decreased or increased, respectively. Lower and upper bounds, respectively denoted as $\text{L}_\text{N}$ and $\text{U}_\text{N}$, are also considered for decreasing and increasing neighborhood in order to prevent unbounded changes.

It is worth to mention that adaptive neighborhood does have strong impact especially in high-dimensional search spaces. Therefore, in problems with small number of dimensions, the neighborhood can be set fixed in the iterations for the sake of simplicity.

\subsubsection{Number of iterations in local search}

This fact should be considered that by going forward in the algorithm, we expect the algorithm to get closer to the actual global best; therefore, the number of local iterations can be decreased in every global iteration,
\begin{equation}
\text{IL}_i \leftarrow \text{max}(\text{IL}_i - \text{S}_\text{IL}, \text{L}_\text{IL}),
\end{equation}
where $\text{S}_\text{IL}$ and $\text{L}_\text{IL}$ are, respectively, the step and lower bound of changing IL (see line \ref{algorithm1_changeIL} of algorithm).
This step is performed to progress the speed of algorithm. The inspiration of this step is that as a Gammarus gets closer to the sea edge which is a representative of GB, the sand gets drier and harder to get in. Gammarus is willing to go inside of the sand to forage in it. Thus, by getting closer to the sea edge, it can go less into sand because of its more dryness. Moreover, there are more nutrients in the sand of sea edge and thus the Gammarus does not need to search a lot.

\begin{figure*}[!t]
\centering
\includegraphics[width=5in]{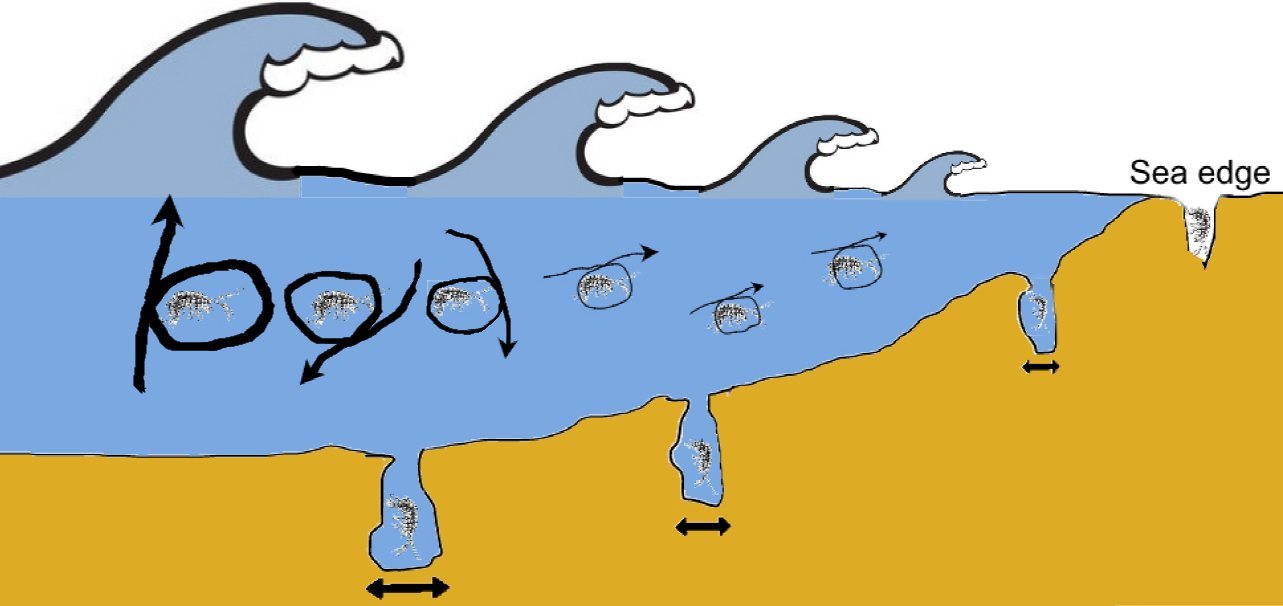}
\caption{Sea waves, sea edge, sea bed, and Gammarus creatures in sea and sea edge.}
\label{fig_sea}
\end{figure*}

\subsection{Sea waves}

\subsubsection{Natural inspiration}

Figure \ref{fig_sea} illustrates a sea with its sea edge and sea bed. As can be seen in this figure, the farther the Gammarus is from the sea edge, the stronger the sea wave becomes. In PMSO algorithm, the sea edge is modeled by global best found so far denoted by GB. Hence, a Gammarus which has found GB, has reached the sea edge and can rest in peace from the sea waves and seek for the nutrients in sea edge. The fact that lots of Gammarus creatures can be seen in in some particular regions of Caspian sea edges supports the claim that Gammarus individuals want to reach the sea edge and rest from the sea waves.

Moreover, as can be seen in figure \ref{fig_sea}, the Gammarus creatures which are close to sea edge are affected by the relatively weak sea waves toward the sea edge. However, far enough from the sea edge, the sea currents behave randomly in direction, although their strengths are impacted by the strength of sea waves. Therefore, the Gammarus particles, far away from the sea edge, are moved randomly but with the strength proportional to their distance from sea edge.

\subsubsection{Exploration}

In every global iteration excluding the first one, for the sake of updating locations of Gammarus population and exploration, they should be shaken in the landscape. Different metaheuristic algorithms perform this step in different ways. In PMSO algprithm, this step is inspired by occurrence of the sea wave as explained before. 
According to distance of each Gammarus from the global best found so far, the strength of wave is determined. The lower the distance gets, the less the strength of wave becomes (see line \ref{algorithm1_waveStrength} in algorithm), 
\begin{equation}
\label{equation5}
|\text{W}_i| \propto |\text{GB} - \text{G}_i|.
\end{equation}
This moves the more distant Gammarus individuals more, which is expected because the Gammarus individuals which are so far from GB are better to be shaken more in comparison to particles close to GB.

As explained before, inspired by nature, a number of Gammarus individuals, which are close to GB, are affected by waves toward GB. The number of particles which are considered to be close to GB, denoted as NC, is a hyper-parameter and should be determined according to optimization problem. For determining whether a Gammarus is close to GB or not, the distances of all Gammarus individuals are computed from GB, and thereafter the distances are sorted in ascending order. The several first number of particles are considered to be close to GB and the rest are far from it. This number is also a hyper-parameter.

The wave angle (direction), affecting the Gammarus individuals which are close to GB, is set in this way:
Suppose for $i^{th}$ Gammarus individual, $V_i$ is the vector connecting $\text{G}_i$ to GB, which means $V_i = \text{GB} - \text{G}_i$, and $V_i(j)$ denotes the component in the $j^{th}$ dimension of this vector. Notice that $j = \{1,\dots,D\}$. The $(D-1)$ number of angles of this vector in different dimensions (the angles of sea wave denoted by $\Theta_i(d)$) can be computed as \cite{blumenson1960derivation,Shih2014ndimensional},   
\begin{equation}\label{equation_angleTowardGlobalBest}
\left\{
      \begin{array}{l}
        \Theta_i(d) = \text{tan}^{-1} \frac{\sqrt{\sum_{j=d+1}^D (V_i(j))^2}}{V_i(d)}, \quad d = \{1,\dots,D-2\}\\
        \\
        \Theta_i(D-1) = 2 \times \text{tan}^{-1} \frac{V_i(D)}{V_i(D-1) + \sqrt{(V_i(D))^2 + (V_i(D-1))^2}}.
      \end{array}
    \right.
\end{equation}

On the other hand, the Gammarus individuals, which are far away from GB, are affected by sea waves with random directions as its natural inspiration was mentioned previously. Hence, for these particles, the $(D-1)$ number of angles of sea wave are computed as,
\begin{equation}\label{equation_angleRandom}
\Theta_i(d) = U(0,2\pi), \quad d = \{1,\dots,D-1\}.
\end{equation}
This moves the more distant Gammarus individuals more, which is beneficial because the Gammarus individuals which are so far from GB are better to be used for exploring other parts of landscape rather than going toward GB.

After computing the strength and direction of sea wave $W_i$ affecting $i^{th}$ Gammarus, the vector of sea wave is found as \cite{blumenson1960derivation,Shih2014ndimensional},
\begin{equation}\label{equation_waveVectorComponents}
\left\{
      \begin{array}{l}
        W_i(1) = |W_i| \text{ cos}(\Theta_i(1)),\\
        \\
        W_i(d) = |W_i| \text{ cos}(\Theta_i(d))\prod_{j=1}^{d-1}\text{sin}(\Theta_i(j))\\
        \quad \quad \quad \quad \quad \quad \quad \quad \quad \quad \quad \quad d = \{2,\dots,D-1\},\\
        \\
        W_i(D) = |W_i| \text{ sin}(\Theta_i(D-1))\prod_{j=1}^{D-2}\text{sin}(\Theta_i(j)).
      \end{array}
    \right.
\end{equation}
\begin{equation}\label{equation_waveVector}
W_i = [W_i(1), W_i(2), \dots, W_i(D)]^T.
\end{equation}
This wave affects all Gammarus individuals except the Gammarus which has found GB so far (see lines \ref{algorithm1_start_updateLocation}-\ref{algorithm1_end_updateLocation} of algorithm). The Gammarus, which is founder of GB, is put in the GB exactly in order to better search around GB for a possible better solution. This fact has a biological support because GB is modeling sea edge in which Gammarus creature rests from the sea waves. 

It is worth to mention that randomly selecting the angles for distant Gammarus individuals helps to escape from the local bests in some cases much better than PSO, in which all particles are moved toward their global and local bests. That is because the particles are given a chance to explore more spaces of landscape and perhaps find the global best in outlying positions. 

\subsection{Initial neighborhood}

\subsubsection{Natural inspiration}

According to figure \ref{fig_sea}, it can be seen that the farther the Gammarus lands on sea bed from the sea edge, the softer the sand is, and thus the Gammarus has more freedom to search for food. Therefore, the neighborhood of Gammarus creature in sea bed is larger in the far distances from sea edge.

\subsubsection{Initial neighborhood determination}

At the start of every global iteration, the neighborhood of each Gammarus individual should be set for the local search which was explained in previous sections. In the first global iteration, all particles are given a fixed pre-defined neighborhood. This initial neighborhood is used for both collision check and local search in the first iteration. In the preceding global iterations, however, the neighborhoods of Gammarus population are determined differently for the local search: At the start of every global iteration, the neighborhood of each Gammarus is set according to its distance from GB. Inspired by nature, as was explained, the more distant the Gammarus is from GB, the larger its initial neighborhood should be. Hence, the initial neighborhood $\text{N}_i$ is determined as,
\begin{equation}\label{equation_neighborhoodAdaptation}
\text{N}_i \leftarrow \text{F} \times |\text{GB} - \text{G}_i|
\end{equation}
where F is a hyper-parameter, and is the fraction of distance from GB to be considered as neighborhood (see line \ref{algorithm1_initialNeighborhood} of algorithm).
This step is performed after applying the sea waves and moving the particles as shown in figure \ref{fig_flowchart}.
After initializing $\text{N}_i$ at the first of each global iteration and after applying sea waves, the neighborhood changes adaptively during local search according to Section \ref{section_adaptive_neighborhood}.

Notice that, as shown in figure \ref{fig_flowchart}, the initial neighborhoods of all particles, except the founder of GB, are determined by equation (\ref{equation_neighborhoodAdaptation}). The initial neighborhood of founder of GB at the first of each global iteration is set to be $\text{N}_{\text{GB}}$ which is a pre-defined hyper-parameter.
Moreover, note that similar to adaptive neighborhood property, determining initial neighborhood according to distance from GB has more impact in high-dimensional search spaces. Thus, if the number of dimensions of search space is not very high, the initial neighborhood of particles can be considered fixed.

\subsection{Updating global best \& checking stop criterion}

At the end of every global iteration, the local best of each Gammarus is compared to the global best so far (GB). If a better answer has been obtained in this global iteration, GB is updated (lines \ref{algorithm1_start_updateGB}-\ref{algorithm1_end_updateGB} of algorithm). 
Thereafter, the stop criteria is checked and if it is reached, algorithm is terminated and the best answer found so far (GB) is returned as the global optima. The stop criterion can be one of these conditions: (I) convergence criteria, which is whether the difference between last two global best answers is small enough, (II) time out criteria, (III) reaching the bound of number of global iterations, or (IV) reaching the maximum number of allowed function evaluations.

\begin{figure*}[!t]
\centering
\includegraphics[width=7in]{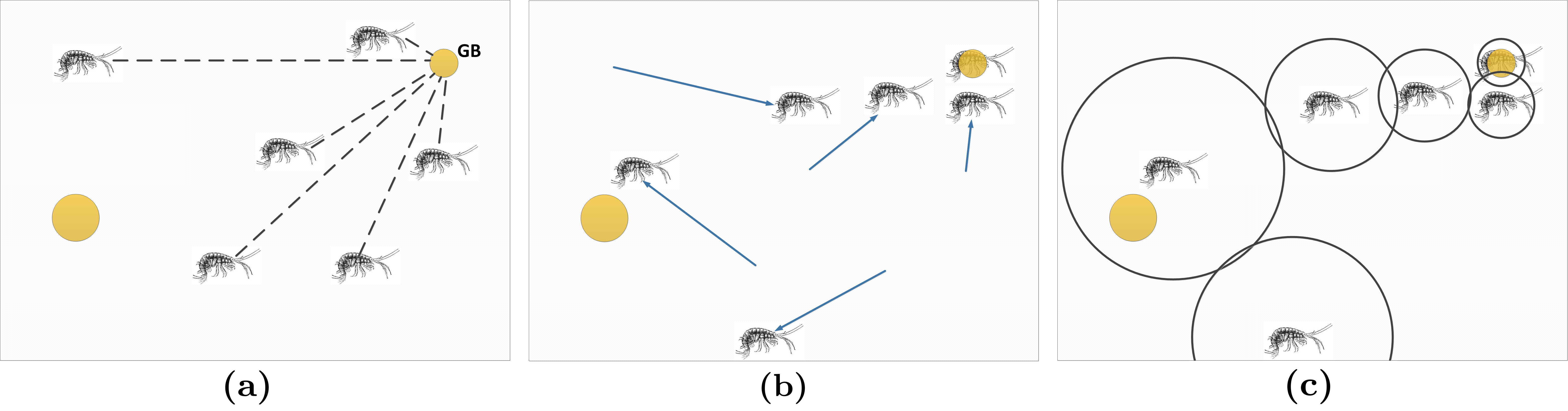}
\caption{A part of a 2D sample scenario in PMSO algorithm.}
\label{fig_scenario}
\end{figure*}

\section{Analysis of algorithm}\label{Section_Analysis}

\subsection{Convergence of particles}

It can be proved that if the best answer found so far is the actual global optima and some particles do not go toward it (especially the ones which are distant from GB and are shaken randomly), after infinite global iterations of algorithm, the particles will converge to it because as they get farther from it, the strength of wave gets bigger. 
\begin{theorem}
If global best found so far, denoted by GB, is actually the global best, all Gammarus individuals converge to it after infinite number of iterations.
\end{theorem}
\begin{proof}
For the Gammarus individuals which are close to sea edge, it is obvious because they are moved toward the GB.
For the distant Gammarus individuals, the following proof is proposed.
Proof by contradiction: As GB is the actual global best, it will not be replaced any more by another location as the global best found so far. Suppose that Gammarus individuals do not converge to GB after infinite number of iterations. If $\text{dist}_i^j$ denotes $|V_i|$ or $|\text{GB} - \text{G}_i|$ in $j^{th}$ global iteration, this means that $\text{dist}_i^{j+1} \geq \text{dist}_i^j$ for every Gammarus individual and in all iterations. In other words, the particles are getting farther and farther from the GB. However, the direction (angle) of sea wave is determined randomly and thus might result Gammarus individual to move toward the GB in one of iterations (e.g. $k^{th}$ iteration). Moreover, as the strength of wave is proportional to $\text{dist}_i^k$, particle will get close to GB in iteration $k$. Thereafter, $\text{dist}_i^{k+1}$ is very small and therefore, the wave affecting that particle will not be strong anymore, and the Gammarus will remain around GB. It means that Gammarus has converged to GB and this contradicts the assumption. 
\end{proof}

Note that there is no worry about the finite number of global iterations because as a new viewpoint, there is no need to collect all the particles to the best answer found so far. The particle which has found GB can search around it without the need of help from all other particles. Having several particles close to GB for helping that particle in searching around GB suffices. Notice that the founder particle of the best answer so far will face a weak wave because it is near the best answer (the sea edge) and can search around the best answer, better. 
In other words, the particle which has found GB in addition to particles close to GB will locally search for fine tuning GB, and the other far particles search far from it in landscape in hope of finding a possible answer which is better than GB and is not found yet. This delicate fact is not considered in algorithms such as PSO, and in our best knowledge, it is not thoroughly analyzed and tackled by other metaheuristic algorithms, too. In other words, different roles are given to the particles of swarm, and the particles cooperate in various roles for the goal of optimization.

\subsection{A sample scenario of algorithm}

A sample scenario is proposed in this section in order to analyze and describe PMSO algorithm. 
Figure \ref{fig_scenario} depicts a part of this sample scenario, in 2D case, for the sake of better visualization. First, the Gammarus individuals are randomly explored in landscape and having a pre-defined similar neighborhoods, they start to search locally while having adaptive neighborhoods. Assume that the local optima at the top-right corner of landscape (shown in figure \ref{fig_scenario}(a)) is found as GB in the first global iteration. Figure \ref{fig_scenario}(a) shows the start of second global iteration. As shown in this figure, the distance of each Gammarus individual is calculated from GB. 

Afterwards, the sea waves are applied to Gammarus particles as shown in figure \ref{fig_scenario}(b). It can be seen in this figure that the Gammarus, which is founder of GB, is put exactly on GB, and is not affected by sea wave. The other two Gammarus individuals, which are close to GB, are affected by waves toward GB; however, the three ones far away from GB are moved in random directions. The strengths of all sea waves, however, are determined according to their distances from GB.

Figure \ref{fig_scenario}(c) shows the step after sea waves being applied. As can be seen in this figure, the neighborhoods of Gammarus particles are set as a fraction of their distance from GB. The farther particles are having wider neighborhood so they can search more distantly in landscape hoping to find a better answer. It can be seen in figure \ref{fig_scenario}(c) that the very far Gammarus at the left does have the chance to find the actual global best because it exists in its neighborhood. Changing neighborhood adaptively will help this particle to probably find the actual global best easier. Adaptation in chaning neighborhood is not shown in figure \ref{fig_scenario} for the sake of brevity. When the left Gammarus finds the actual global best, that point will be GB thereafter, and the sea waves and neighborhoods will be determined according to that point.

As can be seen in figure \ref{fig_scenario}, if the particles are moved toward GB (or a combination of local and global bests), the actual global best might not be found at all. The random wave directions for far particles, proposed in PMSO algorithm, gives the algorithm a chance not to be trapped in local bests.

\begin{figure*}[!t]
\centering
\begin{subfigure}[b]{0.19\textwidth}
\centering
\includegraphics[width=1in]{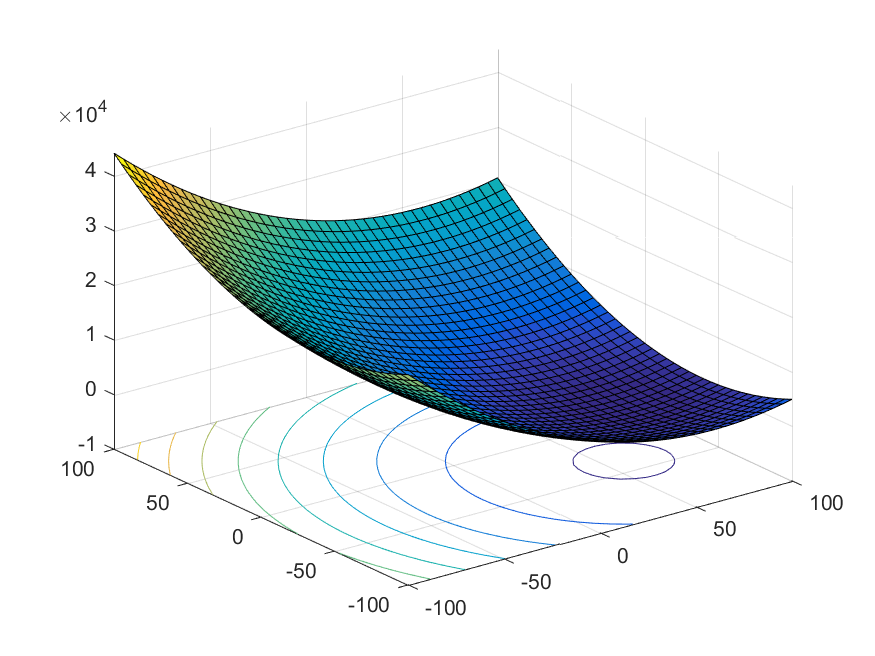} 
\caption{F1}
\end{subfigure}
\begin{subfigure}[b]{0.19\textwidth}
\centering
\includegraphics[width=1in]{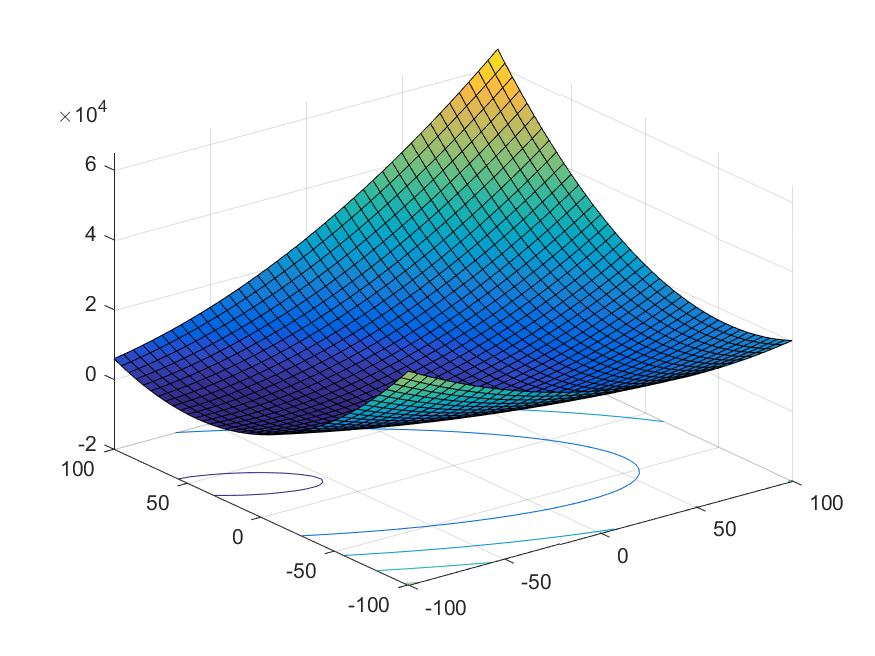}
\caption{F2}
\end{subfigure}
\begin{subfigure}[b]{0.19\textwidth}
\centering
\includegraphics[width=1in]{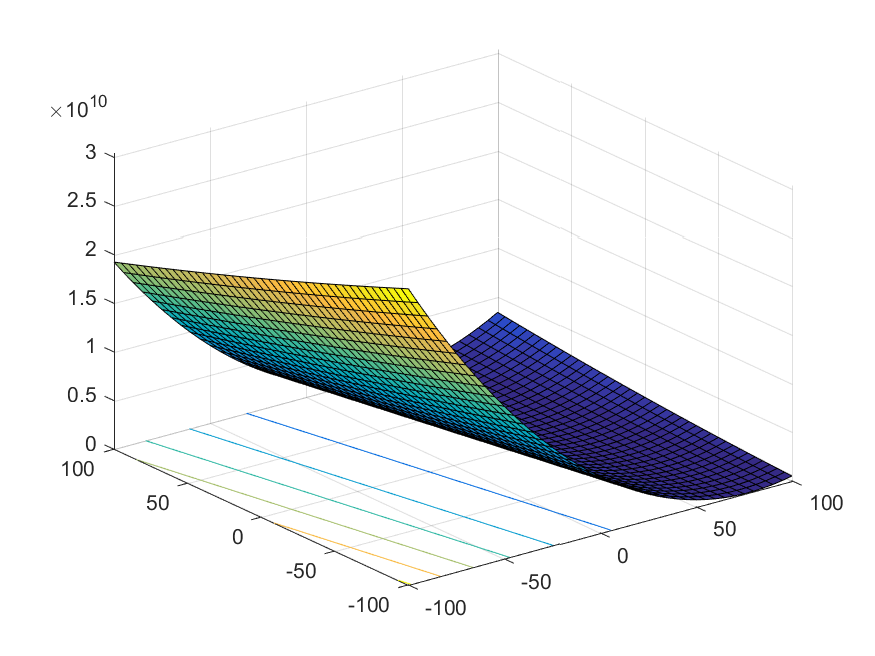}
\caption{F3}
\end{subfigure}
\begin{subfigure}[b]{0.19\textwidth}
\centering
\includegraphics[width=1in]{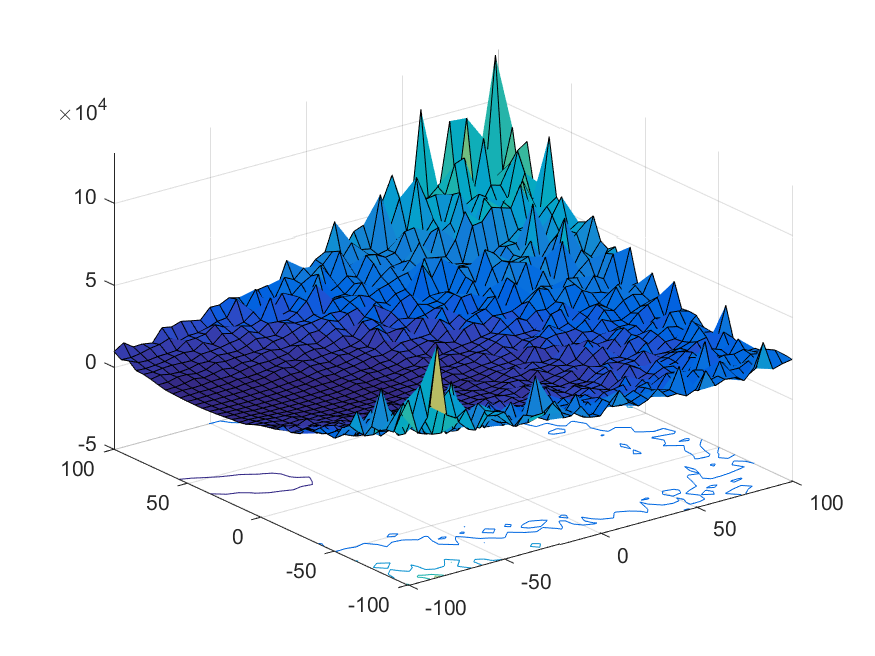}
\caption{F4}
\end{subfigure}
\begin{subfigure}[b]{0.19\textwidth}
\centering
\includegraphics[width=1in]{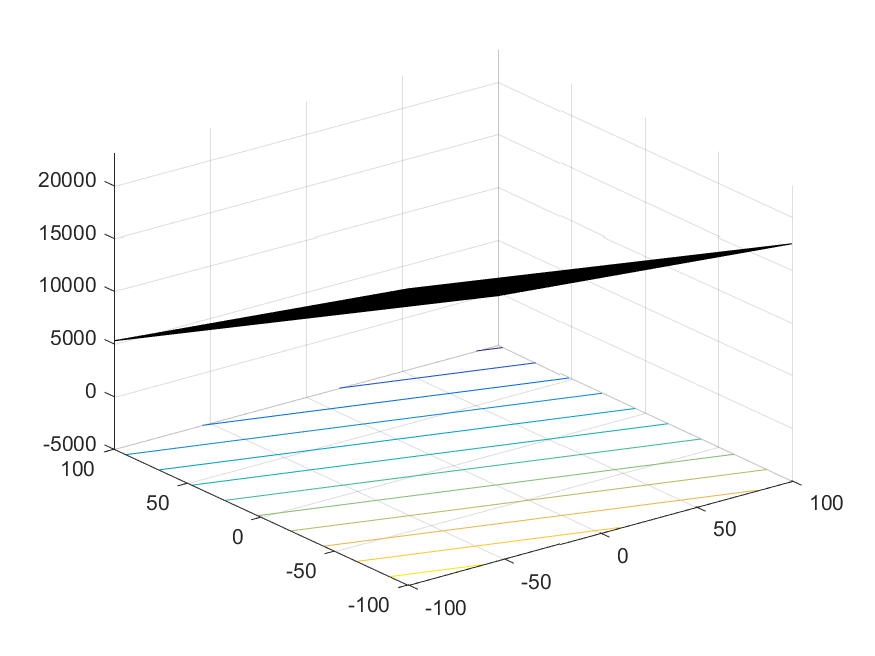}
\caption{F5}
\end{subfigure}
\caption{2D versions of CEC05 uni-modal benchmarks.}
\label{fig_uni_benchmarks}
\end{figure*} 

\begin{figure*}[!t]
\centering
\begin{subfigure}[b]{0.19\textwidth}
\centering
\includegraphics[width=1in]{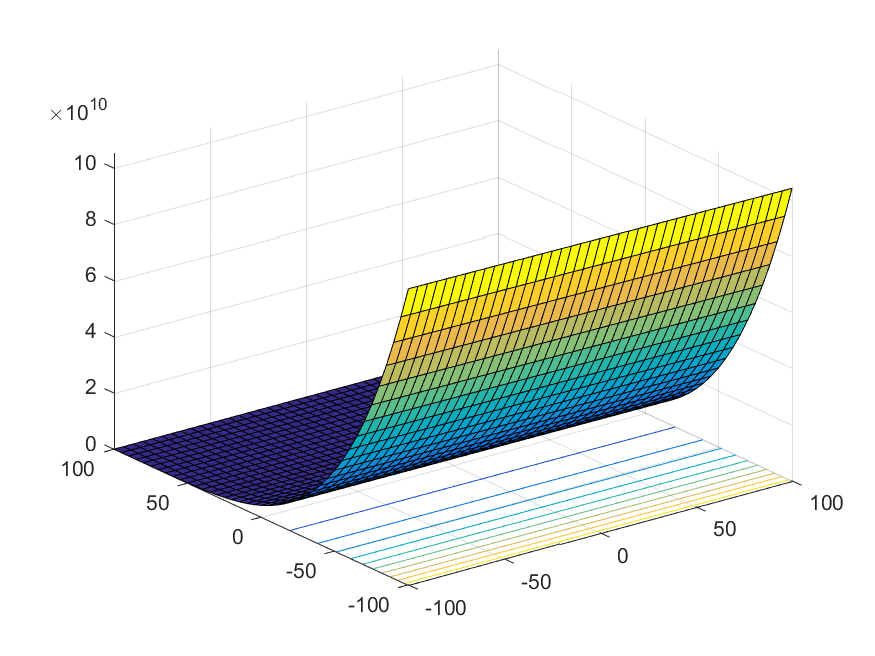} 
\caption{F6}
\end{subfigure}
\begin{subfigure}[b]{0.19\textwidth}
\centering
\includegraphics[width=1in]{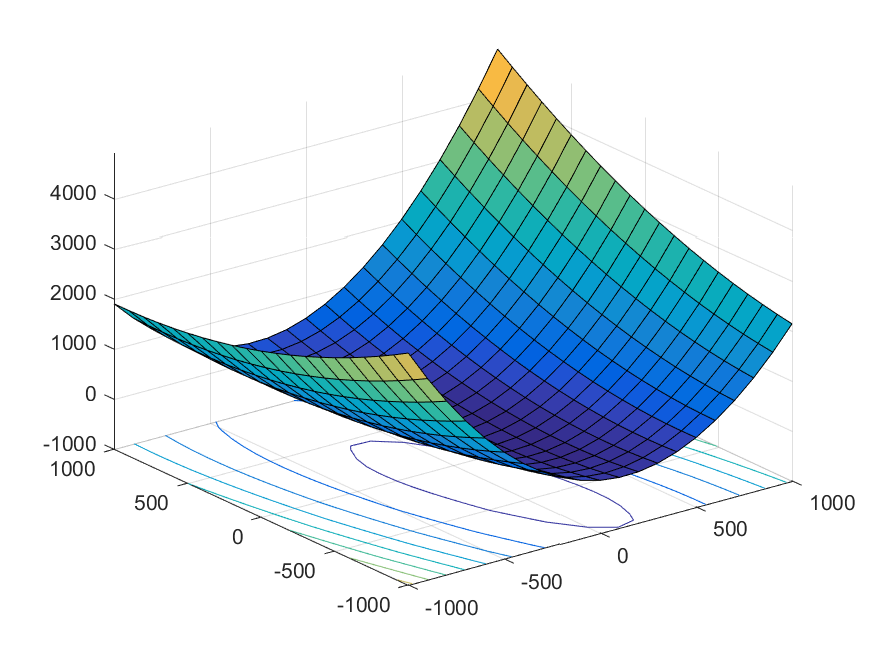}
\caption{F7}
\end{subfigure}
\begin{subfigure}[b]{0.19\textwidth}
\centering
\includegraphics[width=1in]{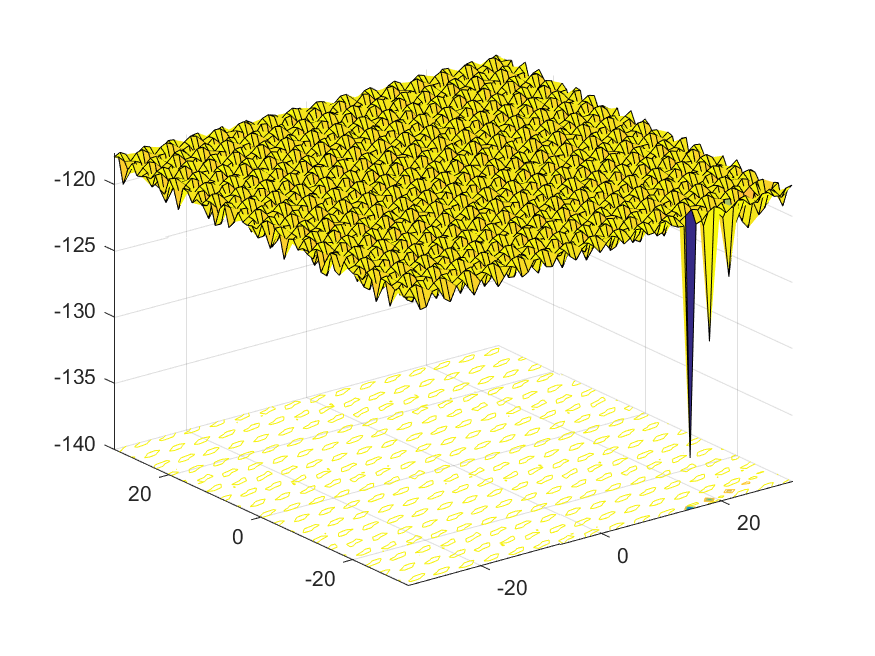}
\caption{F8}
\end{subfigure}
\begin{subfigure}[b]{0.19\textwidth}
\centering
\includegraphics[width=1in]{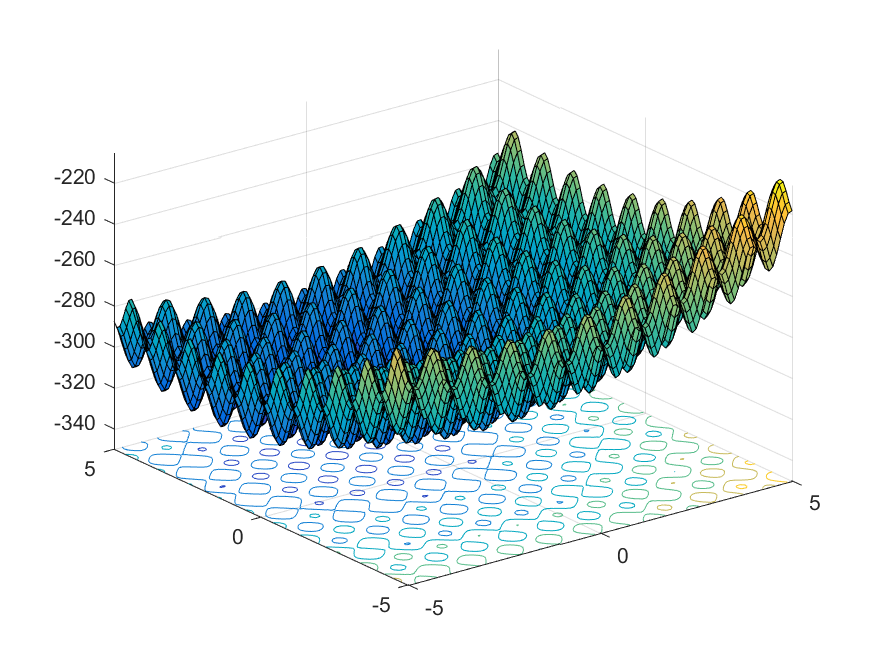}
\caption{F9}
\end{subfigure}
\begin{subfigure}[b]{0.19\textwidth}
\centering
\includegraphics[width=1in]{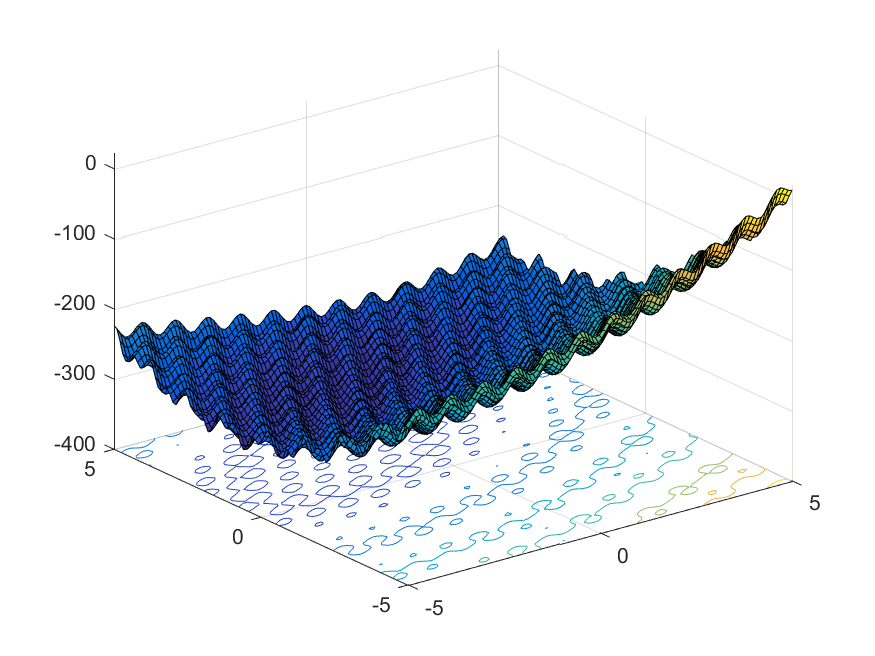}
\caption{F10}
\end{subfigure}
\begin{subfigure}[b]{0.19\textwidth}
\centering
\includegraphics[width=1in]{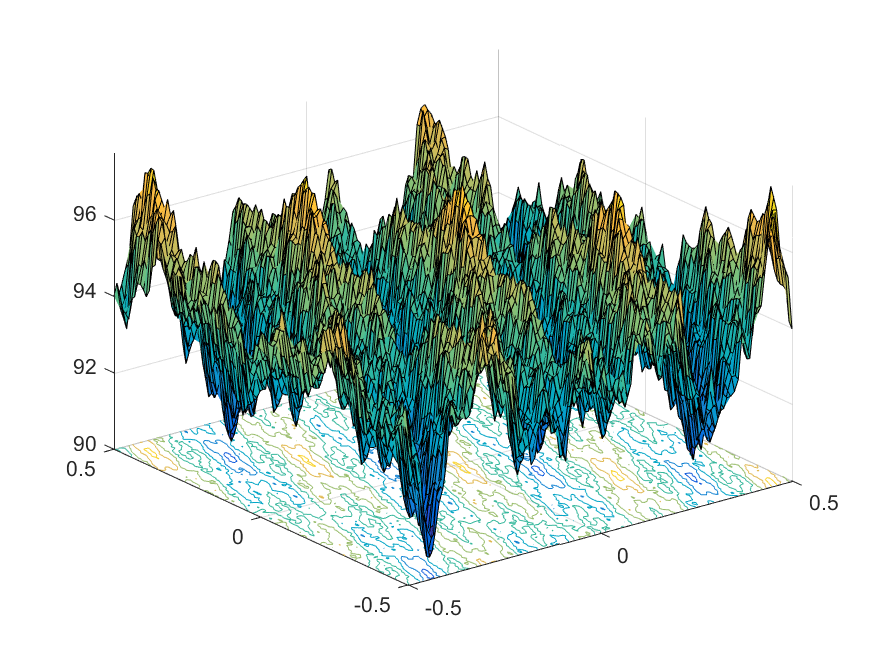}
\caption{F11}
\end{subfigure}
\begin{subfigure}[b]{0.19\textwidth}
\centering
\includegraphics[width=1in]{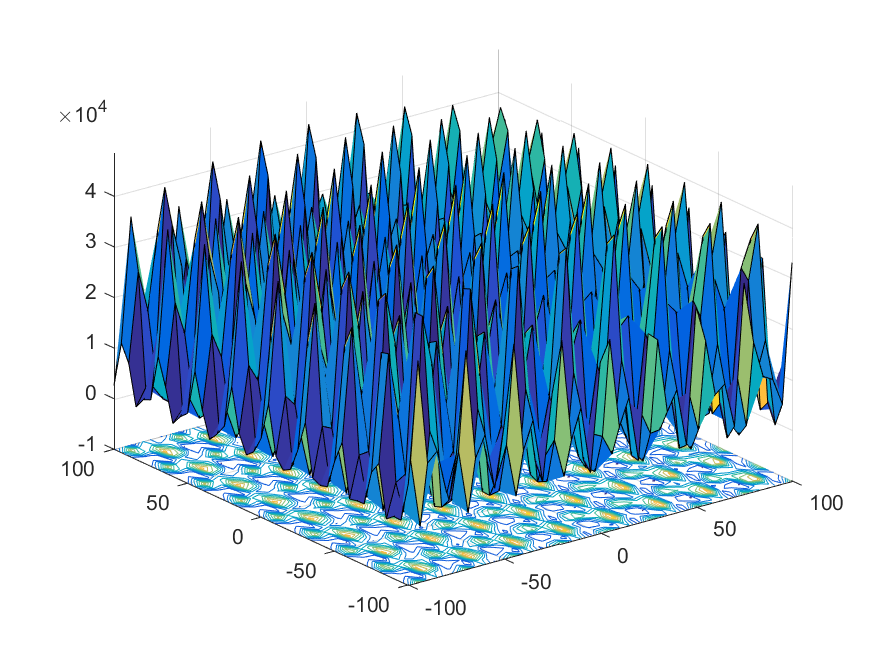}
\caption{F12}
\end{subfigure}
\begin{subfigure}[b]{0.19\textwidth}
\centering
\includegraphics[width=1in]{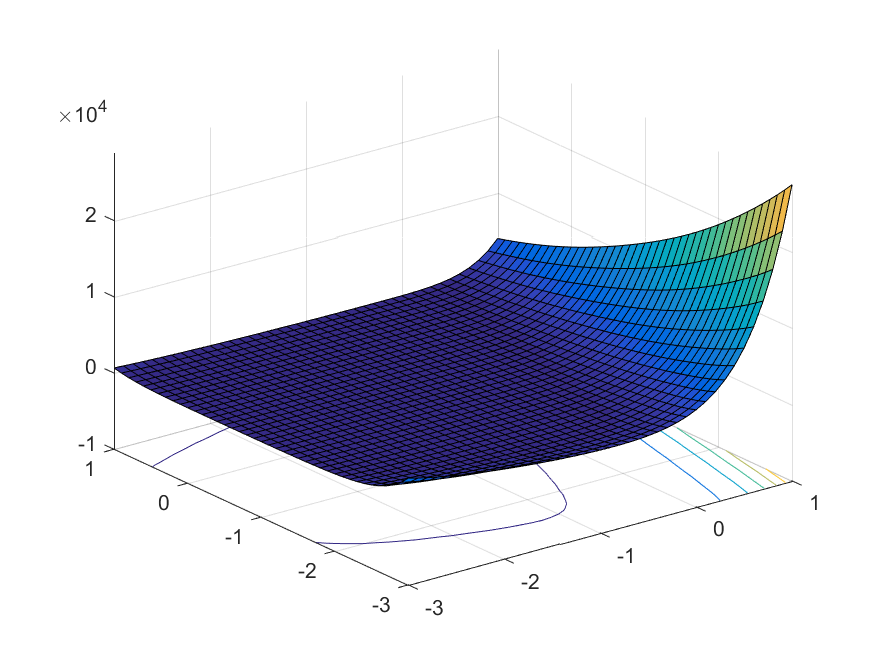}
\caption{F13}
\end{subfigure}
\begin{subfigure}[b]{0.19\textwidth}
\centering
\includegraphics[width=1in]{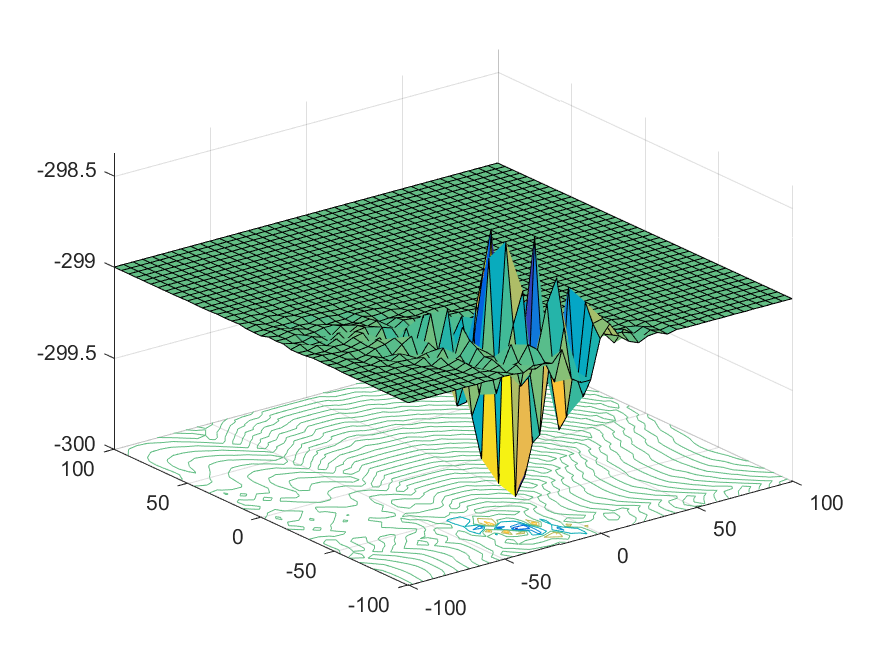}
\caption{F14}
\end{subfigure}
\caption{2D versions of CEC05 multi-modal benchmarks.}
\label{fig_multi_benchmarks}
\end{figure*} 

\begin{figure*}[!t]
\centering
\begin{subfigure}[b]{0.19\textwidth}
\centering
\includegraphics[width=1in]{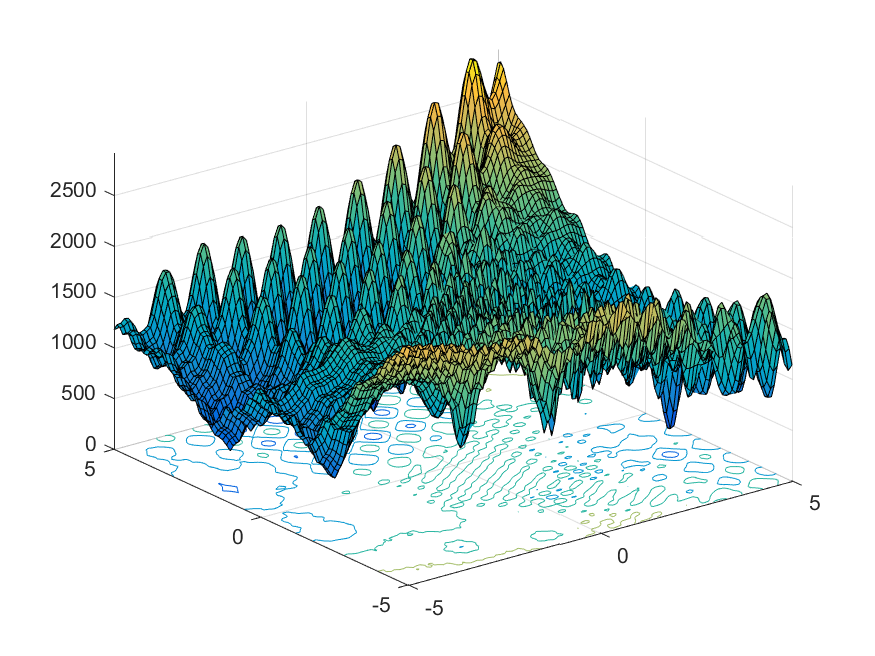} 
\caption{F15}
\end{subfigure}
\begin{subfigure}[b]{0.19\textwidth}
\centering
\includegraphics[width=1in]{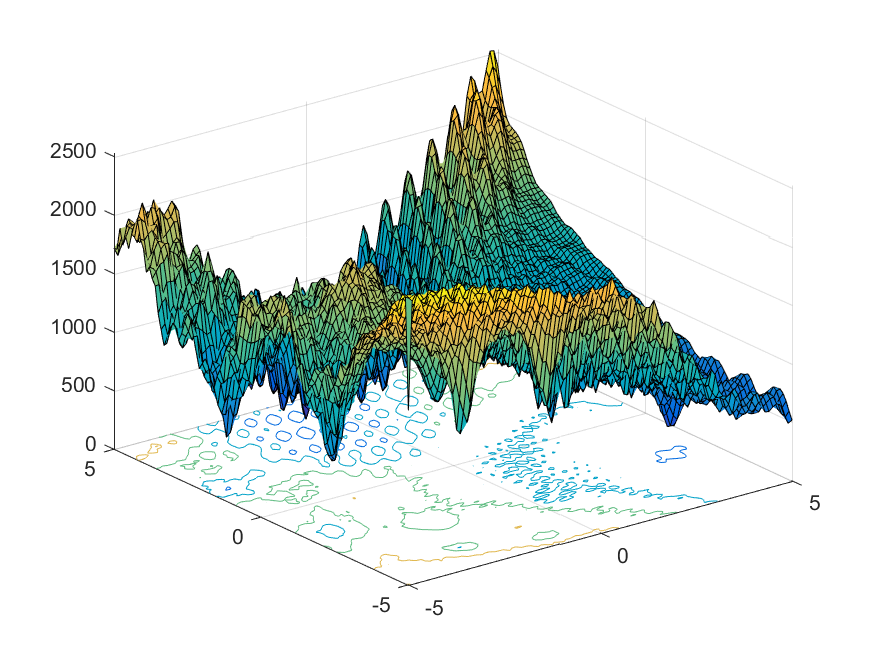}
\caption{F16}
\end{subfigure}
\begin{subfigure}[b]{0.19\textwidth}
\centering
\includegraphics[width=1in]{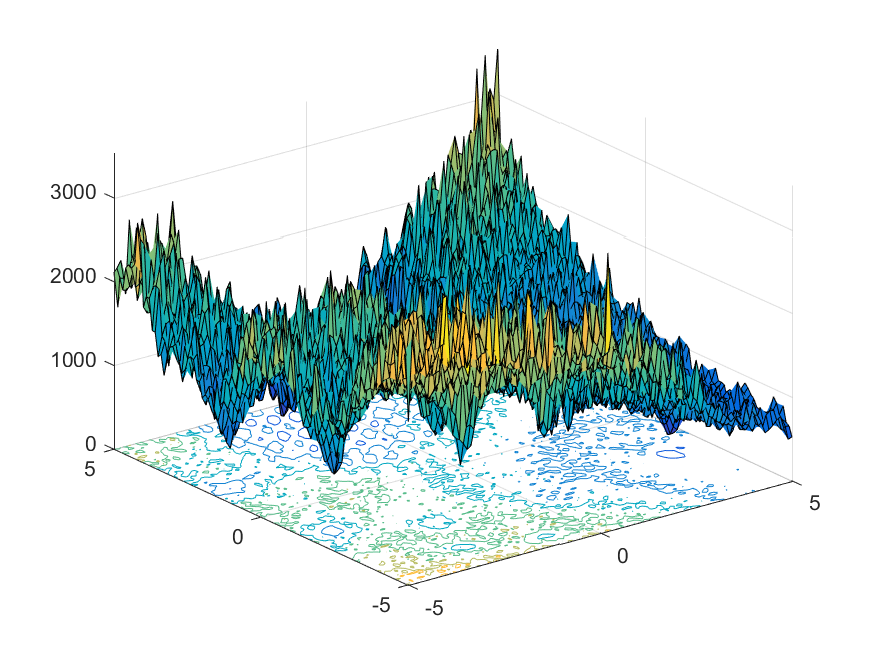}
\caption{F17}
\end{subfigure}
\begin{subfigure}[b]{0.19\textwidth}
\centering
\includegraphics[width=1in]{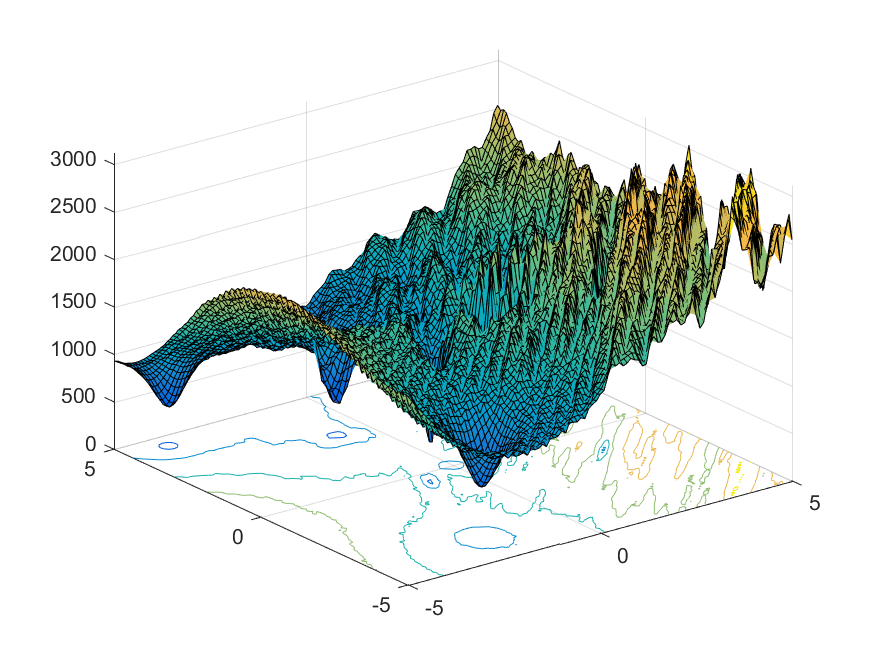}
\caption{F18}
\end{subfigure}
\begin{subfigure}[b]{0.19\textwidth}
\centering
\includegraphics[width=1in]{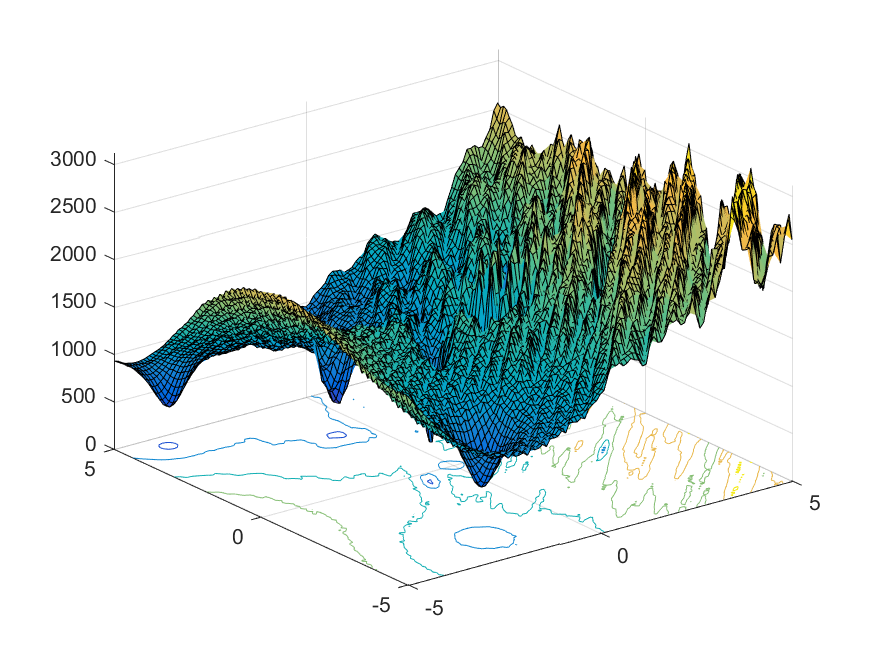}
\caption{F19}
\end{subfigure}
\begin{subfigure}[b]{0.19\textwidth}
\centering
\includegraphics[width=1in]{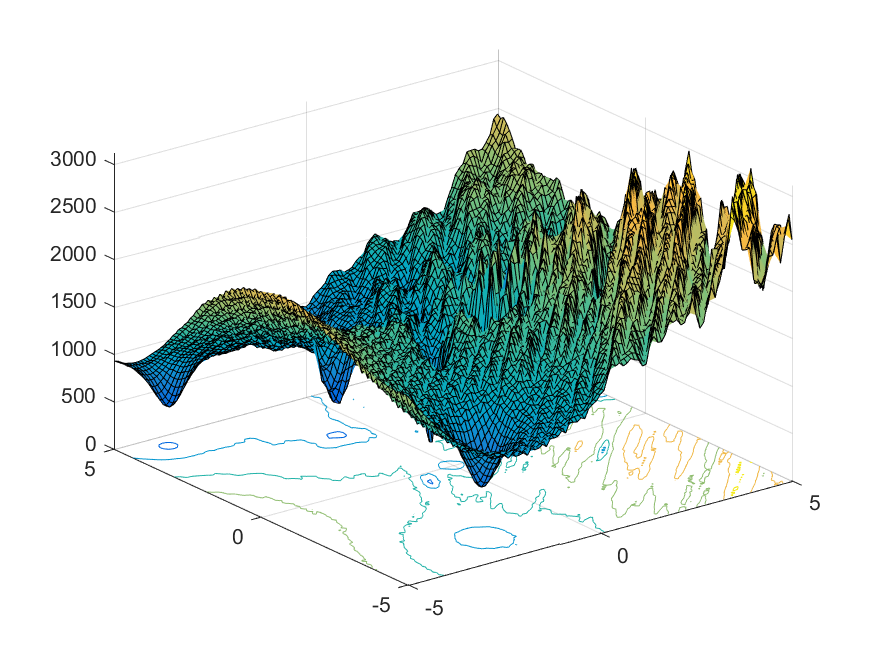}
\caption{F20}
\end{subfigure}
\begin{subfigure}[b]{0.19\textwidth}
\centering
\includegraphics[width=1in]{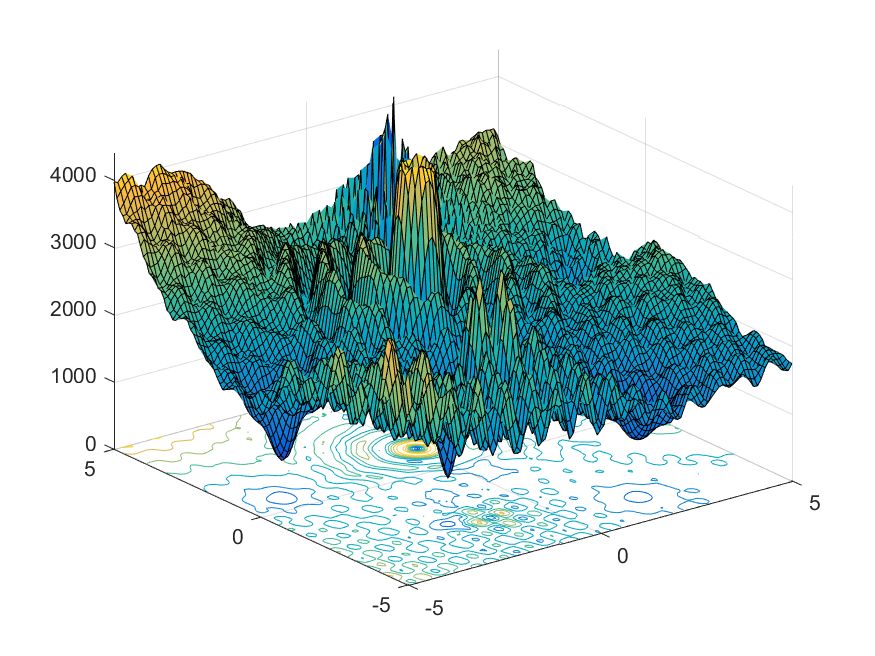}
\caption{F21}
\end{subfigure}
\begin{subfigure}[b]{0.19\textwidth}
\centering
\includegraphics[width=1in]{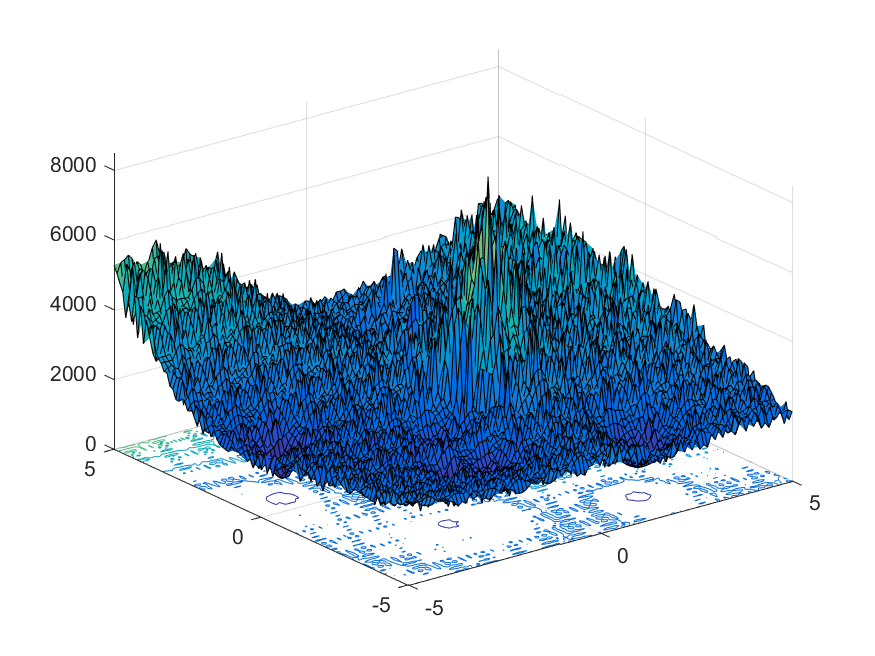}
\caption{F22}
\end{subfigure}
\begin{subfigure}[b]{0.19\textwidth}
\centering
\includegraphics[width=1in]{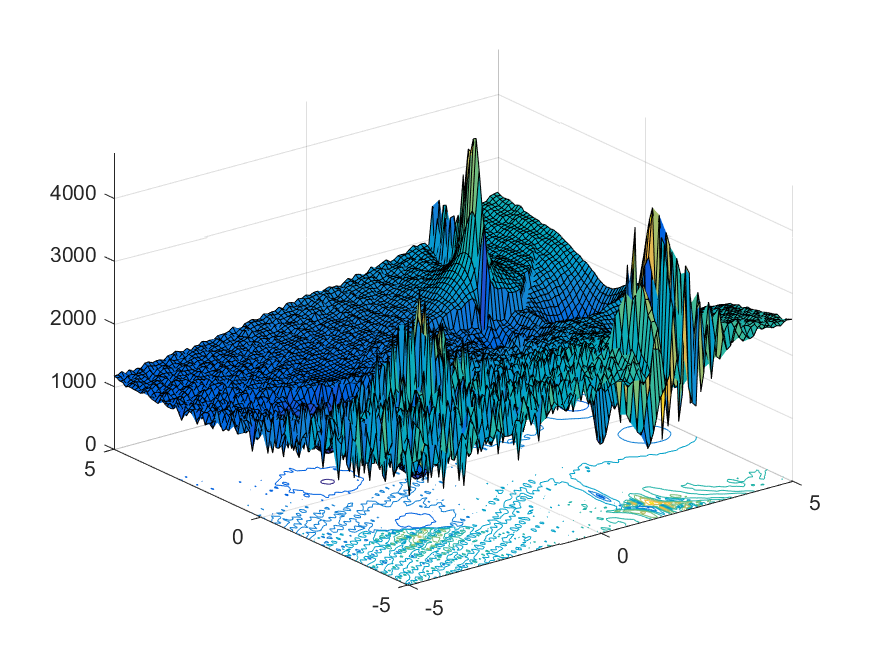}
\caption{F24}
\end{subfigure}
\begin{subfigure}[b]{0.19\textwidth}
\centering
\includegraphics[width=1in]{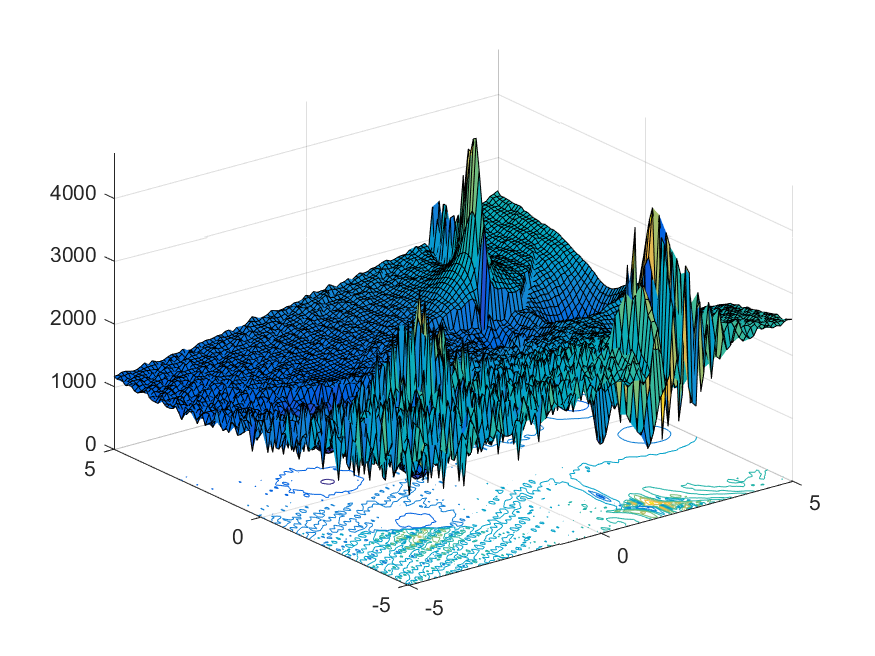}
\caption{F25}
\end{subfigure}
\caption{2D versions of CEC05 hybrid benchmarks.}
\label{fig_hybrid_benchmarks}
\end{figure*} 

\subsection{Time complexity}

\begin{theorem}
The time complexity of PMSO algorithm is $O(\text{NG} \times \text{IG} \times \text{IL})$ in worst case, where $\text{NG}$, $\text{IG}$, and $\text{IL}$ denote number of Gammarus individuals, number of global iterations, and number of local iterations, respectively.
\end{theorem}
\begin{proof}
According to figure \ref{fig_flowchart} and Algorithm \ref{the_algorithm}, it can be seen that PMSO algorithm consists of $\text{IG}$ global iterations, which is the upper bound and the algorithm might terminate sooner by the stop criteria. In every global iteration, each Gammarus searches locally in its neighborhood. The number of local iterations $\text{IL}$ is decremented in every global iteration, however, as an upper bound we do not consider decrementing it. Hence, the time complexity for every particle is $O(\text{IG} \times \text{IL})$. This procedure is performed by every Gammarus individual, and as the population includes $\text{NG}$ particles, the time complexity of PMSO algorithm is $O(\text{NG} \times \text{IG} \times \text{IL})$.
\end{proof}

\subsection{Space complexity}

\begin{theorem}
The space complexity of PMSO algorithm is $O(\log(\text{NG} \times \text{B}))$ in worst case, where NG and B, respectively, denote the number of Gammarus individuals and the size of buffer. Therefore, the required space for PMSO algorithm is linear to the size of input.
\end{theorem}
\begin{proof}
According to figure \ref{fig_flowchart} and Algorithm \ref{the_algorithm}, it can be seen that PMSO algorithm needs $O(\log(\text{NG}))$ spaces for saving the locations of particles. 
Moreover, for saving $\text{N}_i$, buffer, $\text{LB}_i$, and $\text{St}_i$, it requires $O(\log(\text{NG}))$, $O(\log(\text{NG} \times \text{B}))$, $O(\log(\text{NG}))$, and $O(\log(\text{NG} \times 3))$ spaces, respectively, for NG particles. Note that $3$ in time complexity of $\text{St}_i$ is because the particle can have three possible cases, i.e., NoChange, Increase, Decrease. Notice that the as the algorithm is iterating serially on all the particles in every global iteration, it does not need to save these spaces simultaneously for all particles. In other words, these spaces can be overwritten by particles in the iterations. Therefore, for $\text{N}_i$, buffer, $\text{LB}_i$, and $\text{St}_i$, merely $O(c_1)$, $O(\log(\text{B}))$, $O(c_2)$, and $O(c_3)$ spaces are required, respectively, where $c_1$, $c_2$, and $c_3$ are constants. Other parameters, such as GB, need only constant space $O(c_4)$ to be stored. Thus, the total space needed for PMSO algorithm is $O(\log(\text{NG}) + c_1 + \log(\text{B}) + c_2 + c_3 + c_4) = O(\log(\text{NG} \times \text{B}))$.
\end{proof}

%%%%%%%%%%%%%%%%%%%%% table:
\begin{table*}[!t]
%\centerfloat
\renewcommand{\arraystretch}{2}  %%% each row size
\caption{CEC05 uni-modal benchmarks used for experiments}
\label{table_uni_benchmarks}
\centering
\scalebox{0.9}{    %%% --> for resizing tables
\begin{tabular}{l c c c c}
\hline
\hline
\textbf{Benchmark Function} & \textbf{Bounds} & \textbf{Initialization} & $f_{min}$\\
\hline
$\text{F1}(x) = \sum_{i=1}^D z_i^2 - 450$ & $[-100, 100]$ & $[-100, 100]$ & $-450$\\
\hline
$\text{F2}(x) = \sum_{i=1}^D (\sum_{j=1}^i z_j)^2 - 450$ & $[-100, 100]$ & $[-100, 100]$ & $-450$\\
\hline
$\text{F3}(x) = \sum_{i=1}^D (10^6)^{\frac{i-1}{D-1}} z_i^2 - 450$ & $[-100, 100]$ & $[-100, 100]$ & $-450$\\
\hline
$\text{F4}(x) = (\sum_{i=1}^D (\sum_{j=1}^i z_j)^2) \times (1 + 0.4|N(0,1)|) - 450$ & $[-100, 100]$ & $[-100, 100]$ & $-450$\\
\hline
$\text{F5}(x) = \text{max}(|A_i x - B_i|) - 310$ & $[-100, 100]$ & $[-100, 100]$ & $-310$\\
\hline
\hline
\end{tabular}
}
\end{table*}

%%%%%%%%%%%%%%%%%%%%% table:
\begin{table*}[!t]
%\centerfloat
\renewcommand{\arraystretch}{2}  %%% each row size
\caption{CEC05 multi-modal benchmarks used for experiments}
\label{table_multi_benchmarks}
\centering
\scalebox{0.9}{    %%% --> for resizing tables
\begin{tabular}{l c c c}
\hline
\hline
\textbf{Benchmark Function} & \textbf{Bounds} & \textbf{Initialization} & $f_{min}$\\
\hline
$\text{F6}(x) = \sum_{i=1}^{D-1} (100(z_i^2 - z_{i+1})^2 + (z_i - 1)^2) + 390$ & $[-100, 100]$ & $[-100, 100]$ & $+390$\\
\hline
$\text{F7}(x) = \sum_{i=1}^D \frac{z_i^2}{4000} - \prod_{i=1}^D \text{cos}(\frac{z_i}{\sqrt{i}}) + 1 - 180 $ & No bounds & $[0, 600]$ & $-180$\\
\hline
$\text{F8}(x) = -20 \text{exp}(-0.2\sqrt{\frac{1}{D} \sum_{i=1}^D z_i^2}) - \text{exp}(\frac{1}{D} \sum_{i=1}^D \text{cos}(2\pi z_i)) + 20 + e - 140 $ & $[-32, 32]$ & $[-32, 32]$ & $-140$\\
\hline
$\text{F9}(x) = \sum_{i=1}^D (z_i^2 - 10 \text{cos}(2\pi z_i) + 10) - 330 $ & $[-5, 5]$ & $[-5, 5]$ & $-330$\\
\hline
$\text{F10}(x) = \sum_{i=1}^D (z_i^2 - 10 \text{cos}(2\pi z_i) + 10) - 330 $ & $[-5, 5]$ & $[-5, 5]$ & $-330$\\
\hline
$\text{F11}(x) = \sum_{i=1}^D (\sum_{k=0}^{20}[0.5^k \text{cos}(2\pi \times 3^k (z_i + 0.5))]) - D \sum_{k=0}^{20} [0.5^k \text{cos}(2\pi \times 3^k \times 0.5)] + 90 $ & $[-0.5, 0.5]$ & $[-0.5, 0.5]$ & $+90$\\
\hline
$\text{F12}(x) = \sum_{i=1}^D (A_i - B_i(x))^2 - 460 $ & $[-\pi, \pi]$ & $[-\pi, \pi]$ & $-460$\\
\hline
$\text{F13}(x) = F_a(F_b(z_1,z_2)) + F_a(F_b(z_2,z_3)) + \dots + F_a(F_b(z_{D-1},z_D)) + F_a(F_b(z_D,z_1)) - 130$ & $[-3, 1]$ & $[-3, 1]$ & $-130$\\
\hline
$\text{F14}(x) = F_c(z_1, z_2) + F_c(z_2, z_3) + \dots + F_c(z_{D-1}, z_D) + F_c(z_D, z_1) - 300$ & $[-100, 100]$ & $[-100, 100]$ & $-300$\\
\hline
\hline
\end{tabular}
}
\end{table*}

%%%%%%%%%%%%%%%%%%%%% table:
\begin{table*}[!t]
%\centerfloat
\renewcommand{\arraystretch}{2}  %%% each row size
\caption{CEC05 hybrid benchmarks used for experiments}
\label{table_hybrid_benchmarks}
\centering
\scalebox{0.9}{    %%% --> for resizing tables
\begin{tabular}{l c c c}
\hline
\hline
\textbf{Benchmark Function} & \textbf{Bounds} & \textbf{Initialization} & $f_{min}$\\
\hline
$\text{F15}(x) = $ basic functions: Rastrigin, Weierstrass, Griewank, Ackley, Sphere & $[-5, 5]$ & $[-5, 5]$ & $+120$\\
\hline
$\text{F16}(x) = $ $\text{F15}(x)$ with different linear transformation matrices & $[-5, 5]$ & $[-5, 5]$ & $+120$\\
\hline
$\text{F17}(x) = (\text{F16}(x) - 120) \times (1 + 0.2|N(0,1)|) + 120$ & $[-5, 5]$ & $[-5, 5]$ & $+120$\\
\hline
$\text{F18}(x) = $ basic functions: Ackley, Rastrigin, Sphere, Weierstrass, Griewank & $[-5, 5]$ & $[-5, 5]$ & $+10$\\
\hline
$\text{F19}(x) = $ $\text{F18}(x)$ with difference in weights and $\lambda_i$ & $[-5, 5]$ & $[-5, 5]$ & $+10$\\
\hline
$\text{F20}(x) = $ $\text{F18}(x)$ with difference in shifting optima & $[-5, 5]$ & $[-5, 5]$ & $+10$\\
\hline
$\text{F21}(x) = $ basic functions: Rotated Expanded Scaffer, Rastrigin, F13, Weierstrass, Griewank & $[-5, 5]$ & $[-5, 5]$ & $+360$\\
\hline
$\text{F22}(x) = $ $\text{F21}(x)$ with different linear transformation matrices & $[-5, 5]$ & $[-5, 5]$ & $+360$\\
\hline
$\text{F24}(x) = $ basic functions: Weierstrass, Scaffer, F13, Ackley, Rastrigin, Griewank, $\dots$ & $[-5, 5]$ & $[-5, 5]$ & $+260$\\
\hline
$\text{F25}(x) = $ $\text{F24}(x)$ with no bounds & No bounds & $[2, 5]$ & $+260$\\
\hline
\hline
\end{tabular}
}
\end{table*}

%%%%%%%%%%%%%%%%%%%%%%%%%%%%%%%%%%%%%%%%%
\section{Experimental results of algorithm}\label{Section_Experiments}

\subsection{Benchmarks}

For evaluations, verification, and comparison of the proposed optimization algorithm with other state-of-the-art methods, CEC05 benchmarks \cite{suganthan2005problem}, which are standard benchmarks for metaheuristic optimization, are used\footnote{The CEC 2005 benchmark functions are available online: http://www.ntu.edu.sg/home/EPNSugan/}. This bank of benchmark includes different types of benchmark functions, i.e., uni-modal, multi-modal, and hybrid. These benchmarks are conducted on several real-world optimization problems \cite{suganthan2005problem}. 

Uni-modal and multi-modal benchmarks include one and multiple global optimums, receptively. Hybrid functions, however, are composed of several different uni-modal and multi-modal functions with different wights and coverage on domain. Tables \ref{table_uni_benchmarks}, \ref{table_multi_benchmarks}, and \ref{table_hybrid_benchmarks} respectively report the mathematical expressions and properties of uni-modal, multi-modal, and hybrid benchmark functions in CEC05 benchmark bank \cite{suganthan2005problem}. The 2D versions of uni-modal, multi-modal, and hybrid benchmarks are shown in figures \ref{fig_uni_benchmarks}, \ref{fig_multi_benchmarks}, and \ref{fig_hybrid_benchmarks}, respectively. Note that F23 is not considered in hybrid benchmarks because it is quantized and not continuous.

In tables \ref{table_uni_benchmarks}, \ref{table_multi_benchmarks}, and \ref{table_hybrid_benchmarks}, $z_i$ is the $i^{th}$ component of vector $z$. This vector is $z=(x-o)$ for functions F1, F2, F4, F9, and is $z=(x-o)+1$ for functions F6, F13, and is $z=(x-o) \times M$ for functions F3, F7, F8, F10, F11, F14, where $o$ is the shifted global optimum and $M$ denotes an orthogonal transformation matrix. In F5, $A_i$ and $B_i$ are $D \times D$ matrix and $D \times 1$ vector, respectively, where $D$ denotes dimension of benchmark. However, in F12, both $A_i$ and $B_i$ are $D \times D$ matrices. In F13, $F_a(x) = \sum_{i=1}^D \frac{x_i^2}{4000} - \prod_{i=1}^D \text{cos}(\frac{x_i}{\sqrt{i}}) + 1$ and $F_b(x) = \sum_{i=1}^{D-1} (100(x_i^2 - x_{i+1})^2 + (x_i - 1)^2)$. In F14, $F_c(x,y) = 0.5 + \frac{\text{sin}^2(\sqrt{x^2 + y^2}) - 0.5}{(1 + 0.001(x^2 + y^2))^2}$.
On the other hand, for hybrid functions listed in table \ref{table_hybrid_benchmarks}, benchmark functions are composed as weighted summation of basic functions $f_i(x)$ according to this formula: $F(x) = \sum_{i=1}^{n} (w_i \times [f'_i((x - o_i) / \lambda_i \times M_i) + \text{bias}_i]) + f_\text{bias}$, where $w_i$ and $\lambda_i$, respectively, denote the weights and stretch/compress amount of basic function in composition. The reader is refered to \cite{suganthan2005problem} for more details on benchmark functions and their parameters.

\begin{table}[!t]
%\begin{minipage}{\textwidth}
\renewcommand{\arraystretch}{1.3}  %%% each row size
\caption{Mean, best and standard deviation values of solutions achieved for CEC05 uni-modal benchmarks taken over 30 runs at 2 dimensions.}
\label{table_uni_results_2D}
\centering
\scalebox{0.9}{    %%% --> for resizing tables
\begin{tabular}{l | c c c c c}
\hline
\hline
 & \textbf{F1} & \textbf{F2} & \textbf{F3} & \textbf{F4} & \textbf{F5}\\
\hline  
Best & 0.000E$+$00 & 0.000E$+$00 & 5.066E$-$07 & 0.000E$+$00 & 1.453E$-$01 \\
Mean & 7.504E$-$06 & 8.143E$-$06 & 3.143E$+$02 & 9.646E$-$06 & 5.242E$-$01 \\
Std. & 1.781E$-$05 & 2.060E$-$05 & 6.663E$+$02 & 2.767E$-$05 & 2.709E$-$01 \\
\hline
\hline
\end{tabular}
}
%\end{minipage}
\end{table}

\begin{table*}[!t]
%\begin{minipage}{\textwidth}
\renewcommand{\arraystretch}{1.3}  %%% each row size
\caption{Mean, best and standard deviation values of solutions achieved for CEC05 multi-modal benchmarks taken over 30 runs at 2 dimensions.}
\label{table_multi_results_2D}
\centering
\scalebox{0.9}{    %%% --> for resizing tables
\begin{tabular}{l | c c c c c c c c c}
\hline
\hline
 & \textbf{F6} & \textbf{F7} & \textbf{F8} & \textbf{F9} & \textbf{F10} & \textbf{F11} & \textbf{F12} & \textbf{F13} & \textbf{F14}\\
\hline  
Best & 0.000E$+$00 & 7.400E$-$03 & 1.120E$-$02 & 0.000E$+$00 & 0.000E$+$00 & 0.000E$+$00 & 0.000E$+$00 & 0.000E$+$00 & 0.000E$+$00 \\
Mean & 7.161E$+$00 & 1.586E$-$01 & 4.367E$+$00 & 6.630E$-$02 & 8.130E$-$02 & 3.290E$-$02 & 2.440E$-$02 & 1.120E$-$03 & 1.240E$-$02 \\
Std. & 1.870E$+$01 & 2.745E$-$01 & 3.776E$+$00 & 2.524E$-$01 & 2.610E$-$01 & 5.000E$-$02 & 4.660E$-$02 & 3.100E$-$03 & 9.500E$-$03 \\
\hline
\hline
\end{tabular}
}
%\end{minipage}
\end{table*}

\subsection{Experiments on 2D benchmarks}

PMSO algorithm is tested on 2-dimensional CEC05 benchmarks in order to experimentally verify the effectiveness and correctness of this new optimization algorithm. 
The population size is set to be 40 as in \cite{el2012performance}, and the number of runs for each benchmark function is 30 as in \cite{uymaz2015artificial}. According to \cite{suganthan2005problem}, stop criterion is maximum number of function evaluations which is $10,000 \times D$ where $D$ is the dimension of benchmark.

The results of these experiments are reported in tables \ref{table_uni_results_2D} and \ref{table_multi_results_2D} for uni-modal and multi-modal benchmarks, respectively. 
As can be seen in these tables, PMSO algorithm works properly for the goal of optimization and its performance is acceptable. The errors on all different benchmarks, except perhaps F3, are significantly small which is expected. These results experimentally verify the proof of correctness of PMSO algorithm, and show that it can be applied on different optimization problems and benchmarks.

\begin{table*}[!t]
%\begin{minipage}{\textwidth}
\renewcommand{\arraystretch}{1.3}  %%% each row size
\caption{Mean, best and standard deviation values of solutions achieved for CEC05 uni-modal benchmarks taken over 30 runs at 10 dimensions.}
\label{table_uni_results_10D}
\centering
\scalebox{0.9}{    %%% --> for resizing tables
\begin{tabular}{l | l | c : c : c : c : c : c : c}
\hline
\hline
& & \textbf{PMSO} & \textbf{BA} & \textbf{ABC} & $\textbf{HS}_\textbf{POP}$ & $\textbf{ACO}_\textbf{R}$ & \textbf{AAA} & \textbf{DE}\\
\hline
\multirow{3}{*}{F1}  
& Best & 5.900E$-$03 & 2.997E$+$00 & 0.000E$+$00 & 0.000E$+$00 & 0.000E$+$00 & 0.000E$+$00 & 0.000E$+$00\\
& Mean & 6.010E$-$02 & 4.233E$+$03 & 1.895E$-$15 & 0.000E$+$00 & 2.274E$-$14 & 0.000E$+$00 & 0.000E$+$00\\
& Std. & 3.450E$-$02 & 3.812E$+$03 & 1.038E$-$14 & 0.000E$+$00 & 3.533E$-$14 & 0.000E$+$00 & 0.000E$+$00\\
\hline
\multirow{3}{*}{F2}  
& Best & 9.400E$-$03 & 4.257E$+$00 & 3.077E$-$01 & 4.607E$-$03 & 0.000E$+$00 & 1.541E$-$08 & 5.013E$+$00\\
& Mean & 1.024E$-$01 & 4.898E$+$03 & 3.322E$+$00 & 2.774E$+$01 & 1.478E$-$13 & 1.184E$-$06 & 1.403E$+$01\\
& Std. & 5.800E$-$02 & 3.547E$+$03 & 2.893E$+$00 & 3.481E$+$01 & 1.263E$-$13 & 1.757E$-$06 & 5.389E$+$00\\
\hline
\multirow{3}{*}{F3}  
& Best & 1.651E$+$04 & 3.169E$+$05 & 1.860E$+$05 & 3.230E$+$04 & 8.126E$+$04 & 5.199E$+$04 & 3.482E$+$05\\
& Mean & 2.866E$+$05 & 1.528E$+$07 & 7.224E$+$05 & 2.196E$+$05 & 2.238E$+$06 & 3.111E$+$05 & 1.753E$+$06\\
& Std. & 2.208E$+$05 & 1.959E$+$07 & 3.527E$+$05 & 1.723E$+$05 & 2.689E$+$06 & 2.441E$+$05 & 6.862E$+$05\\
\hline
\multirow{3}{*}{F4}  
& Best & 1.072E$-$01 & 1.597E$+$03 & 2.894E$+$02 & 8.339E$-$02 & 5.684E$-$14 & 4.191E$-$05 & 5.814E$+$01\\
& Mean & 1.981E$+$03 & 8.885E$+$03 & 1.327E$+$03 & 5.989E$+$01 & 3.752E$-$13 & 8.199E$-$03 & 1.483E$+$02\\
& Std. & 1.610E$+$03 & 4.665E$+$03 & 6.884E$+$02 & 1.074E$+$02 & 3.221E$-$13 & 9.187E$-$03 & 6.227E$+$01\\
\hline
\multirow{3}{*}{F5}  
& Best & 6.038E$+$01 & 8.165E$+$02 & 8.169E$+$00 & 1.855E$-$10 & 3.456E$+$01 & 0.000E$+$00 & 1.206E$-$01\\
& Mean & 1.929E$+$03 & 8.168E$+$03 & 7.750E$+$01 & 5.555E$+$01 & 3.938E$+$02 & 1.601E$-$11 & 3.384E$+$00\\
& Std. & 1.486E$+$03 & 3.946E$+$03 & 8.828E$+$01 & 1.058E$+$02 & 4.884E$+$02 & 3.433E$-$11 & 3.175E$+$00\\
\hline
\hline
\end{tabular}
}
%\end{minipage}
\end{table*}

\subsection{Comparison on 10D benchmarks}
	
For the sake of comparing the proposed optimization algorithm, 10-dimensional uni-modal, multi-modal, and hybrid benchmarks are utilized. Several well-known state-of-the-art methods, which most of them are inspired by foraging behavior of living organisms as in PMSO algorithm, are used for comparison. These algorithms are Bees Algorithm (BA) \cite{pham2011bees,el2012performance}, Artificial Bee Colony (ABC) \cite{karaboga2005idea,akay2009parameter,el2012performance}, Population-based Harmony Search ($\text{HS}_\text{POP}$) \cite{mukhopadhyay2008population}, Ant Colony Optimization ($\text{ACO}_\text{R}$) \cite{socha2008ant,el2012performance}, Artificial Algae Algorithm (AAA) \cite{uymaz2015artificial}, and Differential Evolution (DE) \cite{storn1997differential,el2012performance}. 

The results of experimenting PMSO algorithm as well as state-of-the-art methods on uni-modal, multi-modal, and hybrid benchmarks are reported in tables \ref{table_uni_results_10D}, \ref{table_multi_results_10D}, and \ref{table_hybrid_results_10D}, respectively. The results of other state-of-the-art methods, which are reported in these tables, are taken from \cite{uymaz2015artificial}.
The settings of experiments are the same as mentioned for 2D experiments.
As is obvious in these three tables, PMSO algorithm does have well performance on the different benchmarks. 
It strongly outperforms BA algorithm \cite{pham2011bees,el2012performance} on all benchmarks. 
PMSO outperforms ABC algorithm \cite{karaboga2005idea,akay2009parameter,el2012performance} on benchmarks F2, F3, F7, F8, and F25, and reaches its performance with slightly difference on benchmark F14. 
Moreover, it can be seen that PMSO algorithm outperforms $\text{HS}_\text{POP}$ algorithm \cite{mukhopadhyay2008population} on benchmarks F2, F7, F8, F11, and F21.
In comparison to $\text{ACO}_\text{R}$ algorithm \cite{socha2008ant,el2012performance}, PMSO algorithm outperforms it on benchmarks F3, F8, F11, F12, F14, F15, F16, F21, F24, and F25, and its performance is close to $\text{ACO}_\text{R}$ on benchmarks F17, F18, F19, F20, and F22.
In addition, it can be observed that PMSO method outperforms AAA algorithm \cite{uymaz2015artificial} on benchmarks F3 and F25, and performs closely to it on function F8. 
Finally, it is obvious that PMSO has better performance than DE algorithm \cite{storn1997differential,el2012performance} on benchmarks F2, F3, F8, and F25, and is close to it on benchmark function F14.

The results show the effectiveness and correctness of the proposed metaheuristic optimization algorithm, which can be used in different problems and various benchmarks. The performance of this method on different types of benchmarks, including uni-modal, multi-modal, and hybrid benchmarks, determines the fact that PMSO algorithm is applicable in optimization problems with different situations. The next section shows one of the applications of this algorithm which can be used.

\begin{table*}[!t]
%\begin{minipage}{\textwidth}
\renewcommand{\arraystretch}{1.3}  %%% each row size
\caption{Mean, best and standard deviation values of solutions achieved for CEC05 multi-modal benchmarks taken over 30 runs at 10 dimensions.}
\label{table_multi_results_10D}
\centering
\scalebox{0.9}{    %%% --> for resizing tables
\begin{tabular}{l | l | c : c : c : c : c : c : c}
\hline
\hline
& & \textbf{PMSO} & \textbf{BA} & \textbf{ABC} & $\textbf{HS}_\textbf{POP}$ & $\textbf{ACO}_\textbf{R}$ & \textbf{AAA} & \textbf{DE}\\
\hline
\multirow{3}{*}{F6}  
& Best & 8.488E$+$00 & 7.899E$+$02 & 5.465E$-$02 & 5.533E$-$01 & 3.661E$-$06 & 1.623E$-$06 & 3.020E$-$01\\
& Mean & 2.870E$+$03 & 2.936E$+$08 & 1.491E$+$00 & 2.709E$+$01 & 8.611E$+$01 & 1.219E$+$00 & 3.033E$+$00\\
& Std. & 3.270E$+$03 & 4.188E$+$08 & 2.487E$+$00 & 2.839E$+$01 & 4.412E$+$02 & 1.491E$+$00 & 2.213E$+$00\\
\hline
\multirow{3}{*}{F7}  
& Best & 1.336E$+$00 & 1.560E$+$02 & 1.267E$+$03 & 1.267E$+$03 & 3.161E$-$01 & 1.267E$+$03 & 2.363E$-$01\\
& Mean & 4.290E$+$01 & 1.645E$+$03 & 1.267E$+$03 & 1.267E$+$03 & 8.581E$-$01 & 1.267E$+$03 & 4.106E$-$01\\
& Std. & 5.009E$+$01 & 7.414E$+$02 & 2.313E$-$13 & 1.094E$-$02 & 2.913E$-$01 & 3.288E$-$02 & 9.131E$-$02\\
\hline
\multirow{3}{*}{F8}  
& Best & 2.006E$+$01 & 2.017E$+$01 & 2.019E$+$01 & 2.018E$+$01 & 2.014E$+$01 & 2.002E$+$01 & 2.021E$+$01\\
& Mean & 2.028E$+$01 & 2.034E$+$01 & 2.033E$+$01 & 2.033E$+$01 & 2.035E$+$01 & 2.016E$+$01 & 2.040E$+$01\\
& Std. & 9.470E$-$02 & 7.941E$-$02 & 7.863E$-$02 & 5.667E$-$02 & 8.296E$-$02 & 8.695E$-$02 & 6.267E$-$02\\
\hline
\multirow{3}{*}{F9}  
& Best & 1.492E$+$01 & 1.766E$+$01 & 0.000E$+$00 & 0.000E$+$00 & 2.985E$+$00 & 0.000E$+$00 & 0.000E$+$00\\
& Mean & 3.678E$+$01 & 5.135E$+$01 & 0.000E$+$00 & 2.801E$-$07 & 7.735E$+$00 & 0.000E$+$00 & 0.000E$+$00\\
& Std. & 1.232E$+$01 & 1.820E$+$01 & 0.000E$+$00 & 1.534E$-$06 & 3.603E$+$00 & 0.000E$+$00 & 0.000E$+$00\\
\hline
\multirow{3}{*}{F10}  
& Best & 2.290E$+$01 & 3.757E$+$01 & 1.008E$+$01 & 1.751E$+$01 & 7.091E$+$00 & 4.975E$+$00 & 1.144E$+$01\\
& Mean & 4.913E$+$01 & 7.113E$+$01 & 2.518E$+$01 & 2.193E$+$01 & 2.340E$+$01 & 1.501E$+$01 & 1.911E$+$01\\
& Std. & 1.518E$+$01 & 2.512E$+$01 & 7.635E$+$00 & 2.371E$+$00 & 7.573E$+$00 & 5.593E$+$00 & 3.599E$+$00\\
\hline
\multirow{3}{*}{F11}  
& Best & 3.974E$+$00 & 6.003E$+$00 & 4.175E$+$00 & 8.113E$+$00 & 4.981E$+$00 & 1.499E$+$00 & 4.875E$+$00\\
& Mean & 7.106E$+$00 & 9.360E$+$00 & 5.415E$+$00 & 9.221E$+$00 & 8.604E$+$00 & 3.667E$+$00 & 6.102E$+$00\\
& Std. & 1.496E$+$00 & 1.518E$+$00 & 7.297E$-$01 & 4.890E$-$01 & 9.727E$-$01 & 9.292E$-$01 & 6.214E$-$01\\
\hline
\multirow{3}{*}{F12}  
& Best & 4.444E$+$02 & 6.907E$+$03 & 8.504E$+$01 & 4.466E$+$01 & 1.349E$+$04 & 2.439E$+$00 & 1.715E$+$02\\
& Mean & 3.558E$+$03 & 3.014E$+$04 & 3.070E$+$02 & 3.168E$+$03 & 2.923E$+$04 & 5.435E$+$02 & 4.341E$+$02\\
& Std. & 4.139E$+$03 & 8.858E$+$03 & 1.634E$+$02 & 3.135E$+$03 & 6.722E$+$03 & 7.766E$+$02 & 1.850E$+$02\\
\hline
\multirow{3}{*}{F13}  
& Best & 8.125E$-$01 & 4.722E$+$00 & 3.125E$-$02 & 2.048E$-$01 & 8.360E$-$01 & 1.228E$-$01 & 7.506E$-$02\\
& Mean & 2.974E$+$00 & 9.636E$+$00 & 2.241E$-$01 & 8.897E$-$01 & 1.692E$+$00 & 4.231E$-$01 & 2.936E$-$01\\
& Std. & 1.717E$+$00 & 3.494E$+$00 & 8.985E$-$02 & 4.433E$-$01 & 5.347E$-$01 & 1.329E$-$01 & 1.176E$-$01\\
\hline
\multirow{3}{*}{F14}  
& Best & 2.867E$+$00 & 3.256E$+$00 & 2.992E$+$00 & 1.170E$+$00 & 3.219E$+$00 & 2.698E$+$00 & 3.174E$+$00\\
& Mean & 3.564E$+$00 & 3.940E$+$00 & 3.412E$+$00 & 2.485E$+$00 & 3.800E$+$00 & 3.296E$+$00 & 3.459E$+$00\\
& Std. & 2.668E$-$01 & 2.307E$-$01 & 1.441E$-$01 & 6.736E$-$01 & 2.861E$-$01 & 2.823E$-$01 & 1.299E$-$01\\
\hline
\hline
\end{tabular}
}
%\end{minipage}
\end{table*}

\begin{table*}[!t]
%\begin{minipage}{\textwidth}
\renewcommand{\arraystretch}{1.3}  %%% each row size
\caption{Mean, best and standard deviation values of solutions achieved for CEC05 hybrid benchmarks taken over 30 runs at 10 dimensions.}
\label{table_hybrid_results_10D}
\centering
\scalebox{0.9}{    %%% --> for resizing tables
\begin{tabular}{l | l | c : c : c : c : c : c : c}
\hline
\hline
& & \textbf{PMSO} & \textbf{BA} & \textbf{ABC} & $\textbf{HS}_\textbf{POP}$ & $\textbf{ACO}_\textbf{R}$ & \textbf{AAA} & \textbf{DE}\\
\hline
\multirow{3}{*}{F15}  
& Best & 1.223E$+$02 & 4.100E$+$02 & 0.000E$+$00 & 0.000E$+$00 & 1.026E$+$02 & 0.000E$+$00 & 4.867E$-$01\\
& Mean & 4.108E$+$02 & 5.921E$+$02 & 7.329E$-$02 & 2.782E$+$02 & 4.355E$+$02 & 3.311E$+$01 & 1.653E$+$01\\
& Std. & 1.144E$+$02 & 9.820E$+$01 & 3.227E$-$01 & 1.786E$+$02 & 1.876E$+$02 & 3.778E$+$01 & 1.813E$+$01\\
\hline
\multirow{3}{*}{F16}  
& Best & 1.166E$+$02 & 1.385E$+$02 & 1.238E$+$02 & 1.159E$+$02 & 1.066E$+$02 & 9.222E$+$01 & 1.107E$+$02\\
& Mean & 1.987E$+$02 & 3.212E$+$02 & 1.476E$+$02 & 1.380E$+$02 & 2.055E$+$02 & 1.293E$+$02 & 1.443E$+$02\\
& Std. & 5.011E$+$01 & 7.830E$+$01 & 1.384E$+$01 & 1.030E$+$01 & 1.178E$+$02 & 1.596E$+$01 & 1.475E$+$01\\
\hline
\multirow{3}{*}{F17}  
& Best & 1.133E$+$02 & 1.659E$+$02 & 1.404E$+$02 & 1.285E$+$02 & 1.215E$+$02 & 9.984E$+$01 & 1.384E$+$02\\
& Mean & 2.065E$+$02 & 3.442E$+$02 & 1.694E$+$02 & 1.507E$+$02 & 1.890E$+$02 & 1.353E$+$02 & 1.730E$+$02\\
& Std. & 4.990E$+$01 & 8.931E$+$01 & 1.529E$+$01 & 1.146E$+$01 & 5.524E$+$01 & 1.825E$+$01 & 1.422E$+$01\\
\hline
\multirow{3}{*}{F18}  
& Best & 4.911E$+$02 & 9.685E$+$02 & 4.035E$+$02 & 6.269E$+$02 & 7.794E$+$02 & 3.000E$+$02 & 5.114E$+$02\\
& Mean & 9.453E$+$02 & 1.104E$+$03 & 5.086E$+$02 & 8.541E$+$02 & 9.317E$+$02 & 4.864E$+$02 & 7.459E$+$02\\
& Std. & 1.360E$+$02 & 5.856E$+$01 & 7.031E$+$01 & 1.013E$+$02 & 9.863E$+$01 & 2.242E$+$02 & 1.012E$+$02\\
\hline
\multirow{3}{*}{F19}  
& Best & 6.150E$+$02 & 9.810E$+$02 & 4.349E$+$02 & 3.002E$+$02 & 8.001E$+$02 & 3.000E$+$02 & 3.740E$+$02\\
& Mean & 9.844E$+$02 & 1.111E$+$03 & 5.213E$+$02 & 8.265E$+$02 & 9.552E$+$02 & 4.529E$+$02 & 6.980E$+$02\\
& Std. & 1.046E$+$02 & 6.753E$+$01 & 7.592E$+$01 & 1.545E$+$02 & 6.679E$+$01 & 1.993E$+$02 & 1.363E$+$02\\
\hline
\multirow{3}{*}{F20}  
& Best & 6.066E$+$02 & 8.019E$+$02 & 5.000E$+$02 & 5.375E$+$02 & 7.244E$+$02 & 3.000E$+$02 & 5.007E$+$02\\
& Mean & 9.610E$+$02 & 1.097E$+$03 & 5.430E$+$02 & 8.674E$+$02 & 9.396E$+$02 & 4.842E$+$02 & 7.536E$+$02\\
& Std. & 1.303E$+$02 & 8.884E$+$01 & 1.075E$+$02 & 1.019E$+$02 & 9.376E$+$01 & 2.135E$+$02 & 1.110E$+$02\\
\hline
\multirow{3}{*}{F21}  
& Best & 2.000E$+$02 & 5.037E$+$02 & 2.035E$+$02 & 5.000E$+$02 & 5.000E$+$02 & 3.000E$+$02 & 2.270E$+$02\\
& Mean & 9.687E$+$02 & 1.266E$+$03 & 3.459E$+$02 & 1.038E$+$03 & 1.088E$+$03 & 5.233E$+$02 & 4.642E$+$02\\
& Std. & 3.105E$+$02 & 1.574E$+$02 & 9.939E$+$01 & 1.990E$+$02 & 2.130E$+$02 & 1.006E$+$02 & 9.040E$+$01\\
\hline
\multirow{3}{*}{F22}  
& Best & 4.573E$+$02 & 9.150E$+$02 & 2.008E$+$02 & 7.574E$+$02 & 5.318E$+$02 & 3.000E$+$02 & 7.760E$+$02\\
& Mean & 8.337E$+$02 & 1.025E$+$03 & 7.084E$+$02 & 7.873E$+$02 & 8.315E$+$02 & 7.347E$+$02 & 7.941E$+$02\\
& Std. & 9.306E$+$01 & 6.621E$+$01 & 1.905E$+$02 & 2.429E$+$01 & 1.010E$+$02 & 1.196E$+$02 & 7.699E$+$00\\
\hline
\multirow{3}{*}{F24}  
& Best & 2.000E$+$02 & 1.110E$+$03 & 2.000E$+$02 & 2.000E$+$02 & 3.744E$+$02 & 2.000E$+$02 & 2.000E$+$02\\
& Mean & 3.386E$+$02 & 1.264E$+$03 & 2.000E$+$02 & 2.400E$+$02 & 6.038E$+$02 & 2.000E$+$02 & 2.000E$+$02\\
& Std. & 2.200E$+$02 & 5.919E$+$01 & 0.000E$+$00 & 1.038E$+$02 & 2.904E$+$02 & 0.000E$+$00 & 1.201E$-$02\\
\hline
\multirow{3}{*}{F25}  
& Best & 2.000E$+$02 & 1.303E$+$03 & 6.176E$+$02 & 2.000E$+$02 & 3.727E$+$02 & 8.120E$+$02 & 8.183E$+$02\\
& Mean & 5.651E$+$02 & 1.390E$+$03 & 7.689E$+$02 & 3.196E$+$02 & 5.708E$+$02 & 8.170E$+$02 & 8.272E$+$02\\
& Std. & 2.963E$+$02 & 4.070E$+$01 & 9.427E$+$01 & 1.165E$+$02 & 2.967E$+$02 & 2.464E$+$00 & 3.619E$+$00\\
\hline
\hline
\end{tabular}
}
%\end{minipage}
\end{table*}

\begin{figure}[!t]
\centering
\includegraphics[width=2.3in]{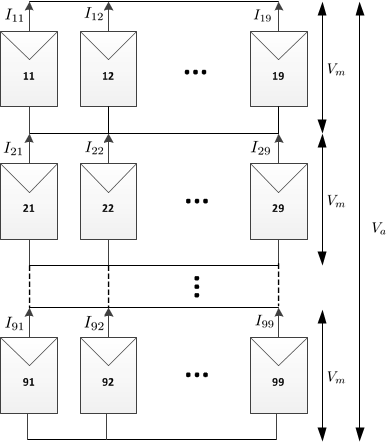}
\caption{A $9 \times 9$ solar PV array in TCT configuration.}
\label{fig_solar_pv_array}
\end{figure}

%%%%%%%%%%%%%%%%%%%%%%%%%%%%%%%%%%%%%%%%%
\section{Using PMSO in partially shaded solar PV array}\label{Section_Solar}

\subsection{Different configurations of solar PV array}
The proposed optimization algorithm can be applied to various types of real-world problems. In this section, as an example, it is used to solve the problem of configuration of partially shaded solar PV array for maximum possible power. Two well-known configurations of PV array are TCT and Su Do Ku configurations. In all configurations of solar PV arrays, each cell can be moved merely on its own column.

In TCT configuration, the solar cells are connected to each other in Series-Parallel (SP) scheme. For instance, figure \ref{fig_solar_pv_array} illustrates a $9 \times 9$ solar PV array in TCT configuration. As can be seen in this figure, the cells in a row are parallel to each other and the cells in a column are connected serially. In this figure, $I_{ij}$ denotes the current of cell in the $i^{th}$ row and the $j^{th}$ column. $V_m$ is also the voltage of every row and $V_a$ is the total obtained voltage which is $9 \times V_m$ here.

However, if shadow falls on the solar array partially, TCT configuration does not disperse the shadow uniformly on the array \cite{deshkar2015solar} and therefore the output power is not the optimum power. Therefore, Su Do Ku arrangement was proposed \cite{woyte2003partial} which disperses the shadow by replacing cells using Su Do Ku puzzle configuration. Yet, the Su Do Ku configuration does not yield the maximum possible power and also has several drawbacks \cite{deshkar2015solar}. 

Hence, recent researches have been done on heuristic optimization for the problem of partial shading in solar arrays. As an example, Genetic algorithm is used for optimization in \cite{deshkar2015solar}. In this paper, the cells are replaced in every column. Every chromosome is $9 \times 9$ two-dimensional and each cell plays the role of a gene. The mutation and cross-over is performed so that every cell can be moved (replaced) in its own column.

\subsection{Calculations of solar PV array}\label{section_Calculations_solar}
Suppose that $I_m$ is the current generated by a cell which receives the standard irradiance $G_0$. $G_0$ is taken to be $G_0 = \text{1000 } W/m^2$ \cite{deshkar2015solar}. 
If the irradiance on the cell is denoted as $G$ ($G \leq G_0$), the irradiance factor is obtained as $k = \frac{G}{G_0}$.

According to Kirchhoff's current law, the total current of each row is obtained as,

\begin{equation}
\label{equation_Kirchhoff_current}
I_{i} = \sum_{j=1}^C k_{ij} I_{ij} \quad \quad ,\forall i \in \{1,2,\dots,R\}
\end{equation}
where $R$ and $C$ are, respectively, the number of rows and columns of array which are both nine in the mentioned example. Also, $k_{ij}$ is the irradiance factor of the cell in the $i^{th}$ row and the $j^{th}$ column. Here, $I_{ij}$ is the current obtained from the corresponding cell at the standard irradiation.
As the cells are supposed to be identical, it can be assumed that all $I_{ij}$'s are the same and are equal to $I_m$. Therefore, equation (\ref{equation_Kirchhoff_current}) can be rewritten as,

\begin{equation}
I_{i} = I_m \sum_{j=1}^C k_{ij} \quad \quad ,\forall i \in \{1,2,\dots,R\}.
\end{equation}

In all the configurations, the row currents are sorted from the smallest to the largest value. Thereafter, the rows are bypassed one-by-one from the first sorted row to the last sorted one, resulting in less amount of output voltage $V_a$. In one of these bypasses, the output power $(P_a = V_a \times I_i)$ is the maximum which is desired. Hence, that bypass is performed and the maximum power of that configuration is obtained. 
In the following sections, results of bypassing rows using different methods, including PMSO, are reported. 

\subsection{PMSO algorithm for solar PV array}
The proposed PMSO algorithm can be utilized for any arbitrary optimization problem. Here, we use this algorithm for partially shaded solar PV array as an example. Because of technical calculations and configurations of sloar PV array, some slight changes are required to be applied on this algorithm which are detailed in the following.

The pseudo-code of proposed PMSO algorithm for this problem is shown in Algorithm \ref{the_algorithm2}. As can be seen, the algorithm is similar to the original proposed algorithm with some differences. In this algorithm, every Gammarus individual is $R \times C$ two-dimensional. $R$ is the number of rows and $C$ is the number of columns of array. Here, in the mentioned example, both $R$ and $C$ are nine. For every particle, the distance from the found global best GB and the location of particle is calculated (line \ref{algorithm2_distance} of Algorithm \ref{the_algorithm2}). In this calculation, $f(\text{GB}, \text{G}_i)$ is the second norm of vector $v$ whose $m^{th}$ component is found as,

\begin{equation}
\label{equation9}
v(m) = 
  \left\{
      \begin{array}{l}
        0 \quad \quad \text{if} \quad \text{GB}(m) = \text{G}_i(m),\\
        1 \quad \quad \text{otherwise}.\\
      \end{array}
    \right.
\end{equation}
Also, $f_\text{max}$ is the second norm of a vector with the size of $v$ and all elements of one. Hence, the calculated distance is normalized and is always less than or equal to one.

\begin{algorithm}[!t]
%\scriptsize
\caption{PMSO for Partially Shaded Solar PV Array}\label{the_algorithm2}
\begin{algorithmic}[1]
\State \textbf{Initialize} NG, IL, IG, and initial $\text{N}_i$
\For{$i$ = 1 to NG}
	\While{collision occurred}
		\State Randomly \textbf{Swap} cells in their columns
	\EndWhile
\EndFor
\While{stop criterion is not reached} 
	\For{$i$ = 1 to NG}
		\If{it is not first iteration}
			\For{$c$ = 1 to C} \label{algorithm2_iteration_columns}
				\State $\text{Distance} \gets \frac{f(\text{GB}, \text{G}_i)}{f_{\text{max}}}$ \label{algorithm2_distance}
				\For{$r$ = 1 to R}  \label{algorithm2_iteration_rows}
					\If{$U(0,1) < \text{Distance}$}  \label{algorithm2_probability_distance}
						\State $|\text{W}_i| \gets \text{round}(R \times \text{Distance})$ \label{algorithm2_wave_magnitude}
						\State $\text{W}_i \gets \text{round}(U(-|\text{W}_i|,|\text{W}_i|))$ \label{algorithm2_wave_angle}
						\If{$r + \text{W}_i < 1$} \label{algorithm2_start_ring}
            				\State \textbf{Swap} cells $(r,c)$ \textbf{and} $((r + \text{W}_i + R),c)$
			        	\ElsIf{$r + \text{W}_i > R$}
    				        \State \textbf{Swap} cells $(r,c)$ \textbf{and} $((r + \text{W}_i - R),c)$
    			    	\Else
	        		    	\State \textbf{Swap} cells $(r,c)$ \textbf{and} $((r + \text{W}_i),c)$
	        			\EndIf  \label{algorithm2_end_ring}
		        	\EndIf
    		    \EndFor
        	\EndFor  
		\EndIf
		\For{$j$ = 1 to IL} 
			\State \textbf{do} Local Search   
		\EndFor
		\If{$\text{LB}_i$ is better than GB} 
			\State GB $\gets \text{LB}_i$
		\EndIf  
	\EndFor
	\State $\text{IL} \gets$ max($\text{IL} - \text{S}_\text{IL}, \text{L}_\text{IL}$) 
\EndWhile
\end{algorithmic}
\end{algorithm}

Afterwards, by iterations on all cells of a Gammarus particle (lines \ref{algorithm2_iteration_columns} and \ref{algorithm2_iteration_rows} of Algorithm \ref{the_algorithm2}), every cell of the Gammarus is moved in its own column with probability of Distance (line \ref{algorithm2_probability_distance} of Algorithm \ref{the_algorithm2}). If it is moved, the amount of moving, which is the wave size $|\text{W}_i|$, is calculated by rounding the product of $R$ and Distance (line \ref{algorithm2_wave_magnitude} of Algorithm \ref{the_algorithm2}); therefore, the wave size is at most the number of rows because Distance is normalized. Then, the wave is a random number in the range $[-|\text{W}_i|,|\text{W}_i|]$ (line \ref{algorithm2_wave_angle} of Algorithm \ref{the_algorithm2}). This models the angle of the wave in the original algorithm. Thereafter, the cell is moved in its own column with this assumption that the first and last cell of every column are connected to each other as a ring (lines \ref{algorithm2_start_ring}-\ref{algorithm2_end_ring} of Algorithm \ref{the_algorithm2}). The rest of the algorithm is the same as the original one, except that initial and adaptive neighborhoods are not used in this algorithm because Gammarus particles are 2D and as was previously explained, in low-dimensional problems, these steps can be omitted for the sake of simplicity.

At every step of algorithm, each Gammarus particle presents a possible configuration of the solar array. Here, the fitness of every Gammarus particle is the maximum output power of the configuration the Gammarus presents. The maximum output power is obtained as was explained in Section \ref{section_Calculations_solar}.

%%%%%%%%%%%%%%%%%%%%% table:
%\newcolumntype{s}{>{\columncolor[HTML]{D0F0C0}} c}
\begin{table}[!t]
%\begin{minipage}{\textwidth}
\renewcommand{\arraystretch}{1.3}  %%% each row size
\caption{Calculations of solar PV array in the configurations of TCT and Su Do Ku}
\label{table_solar_calculations_1}
\centering
\scalebox{0.8}{    %%% --> for resizing tables
\begin{tabular}{l | c | c | c}
\hline
\hline
& Sorted cells to be bypassed & $V_a$ & $P_a$\\
\hline
\multirow{9}{*}{ \textbf{TCT configuration}} 
& $I_9 = 3.6 I_m$ & $9 V_m$ & $32.4 V_m I_m$\\
& $I_8 = 3.6 I_m$ & $8 V_m$ & $28.8 V_m I_m$\\
& $I_7 = 3.6 I_m$ & $7 V_m$ & $25.2 V_m I_m$\\
& $I_6 = 6.6 I_m$ & $6 V_m$ & $39.6 V_m I_m$\\
& $I_5 = 8.1 I_m$ & $5 V_m$ & $40.5 V_m I_m$\\
& $I_4 = 8.1 I_m$ & $4 V_m$ & $32.4 V_m I_m$\\
& $I_3 = 8.1 I_m$ & $3 V_m$ & $24.3 V_m I_m$\\
& $I_2 = 8.1 I_m$ & $2 V_m$ & $16.2 V_m I_m$\\
& $I_1 = 8.1 I_m$ & $V_m$ & $8.1 V_m I_m$\\
\hline
\multirow{9}{*}{ \textbf{Su Do Ku configuration}} 
& $I_6 = 6.3 I_m$ & $9 V_m$ & $56.7 V_m I_m$\\
& $I_7 = 6.3 I_m$ & $8 V_m$ & $50.4 V_m I_m$\\
& $I_8 = 6.3 I_m$ & $7 V_m$ & $44.1 V_m I_m$\\
& $I_1 = 6.3 I_m$ & $6 V_m$ & $37.8 V_m I_m$\\
& $I_2 = 6.3 I_m$ & $5 V_m$ & $31.5 V_m I_m$\\
& $I_4 = 6.6 I_m$ & $4 V_m$ & $26.4 V_m I_m$\\
& $I_3 = 6.6 I_m$ & $3 V_m$ & $19.8 V_m I_m$\\
& $I_5 = 6.6 I_m$ & $2 V_m$ & $13.2 V_m I_m$\\
& $I_9 = 6.6 I_m$ & $V_m$ & $6.6 V_m I_m$\\
\hline
\hline
\end{tabular}%
}
%\end{minipage}
\end{table}

%%%%%%%%%%%%%%%%%%%%% table:
%\newcolumntype{s}{>{\columncolor[HTML]{D0F0C0}} c}
\begin{table}[!t]
%\begin{minipage}{\textwidth}
\renewcommand{\arraystretch}{1.3}  %%% each row size
\caption{Calculations of solar PV array in the configurations obtained by GA and PMSO algorithms}
\label{table_solar_calculations_2}
\centering
\scalebox{0.8}{    %%% --> for resizing tables
\begin{tabular}{l | c | c | c}
\hline
\hline
& Sorted cells to be bypassed & $V_a$ & $P_a$\\
\hline
\multirow{9}{*}{ \textbf{GA algorithm}} 
& $I_2 = 6.3 I_m$ & $9 V_m$ & $56.7 V_m I_m$\\
& $I_5 = 6.3 I_m$ & $8 V_m$ & $50.4 V_m I_m$\\
& $I_6 = 6.3 I_m$ & $7 V_m$ & $44.1 V_m I_m$\\
& $I_8 = 6.3 I_m$ & $6 V_m$ & $37.8 V_m I_m$\\
& $I_3 = 6.4 I_m$ & $5 V_m$ & $32 V_m I_m$\\
& $I_7 = 6.5 I_m$ & $4 V_m$ & $26 V_m I_m$\\
& $I_1 = 6.6 I_m$ & $3 V_m$ & $19.8 V_m I_m$\\
& $I_4 = 6.6 I_m$ & $2 V_m$ & $13.2 V_m I_m$\\
& $I_9 = 6.6 I_m$ & $V_m$ & $6.6 V_m I_m$\\
\hline
\multirow{9}{*}{ \textbf{PMSO algorithm}} 
& $I_1 = 6.2 I_m$ & $9 V_m$ & $56.7 V_m I_m$\\
& $I_5 = 6.2 I_m$ & $8 V_m$ & $50.4 V_m I_m$\\
& $I_2 = 6.3 I_m$ & $7 V_m$ & $44.1 V_m I_m$\\
& $I_8 = 6.3 I_m$ & $6 V_m$ & $37.8 V_m I_m$\\
& $I_3 = 6.4 I_m$ & $5 V_m$ & $32 V_m I_m$\\
& $I_6 = 6.4 I_m$ & $4 V_m$ & $26 V_m I_m$\\
& $I_9 = 6.4 I_m$ & $3 V_m$ & $19.8 V_m I_m$\\
& $I_7 = 6.6 I_m$ & $2 V_m$ & $13.2 V_m I_m$\\
& $I_4 = 7.1 I_m$ & $V_m$ & $6.6 V_m I_m$\\
\hline
\hline
\end{tabular}%
}
%\end{minipage}
\end{table}

\begin{figure*}[!t]
\centering
\begin{subfigure}[b]{0.49\textwidth}
\centering
\includegraphics[width=2.5in]{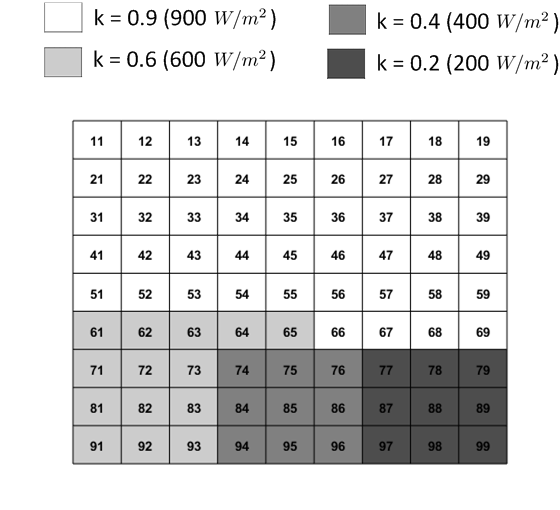} 
\caption{TCT configuration}
\label{fig_TCT_array}
\end{subfigure}
\begin{subfigure}[b]{0.49\textwidth}
\centering
\includegraphics[width=2.5in]{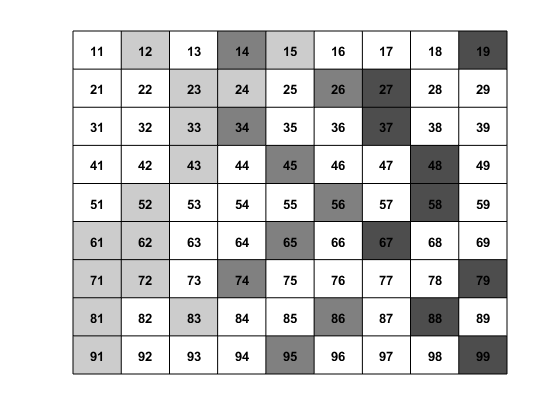}
\caption{Su Do Ku configuration \cite{deshkar2015solar}}
\label{fig_SuDoKu_array}
\end{subfigure}
\begin{subfigure}[b]{0.49\textwidth}
\centering
\includegraphics[width=2.5in]{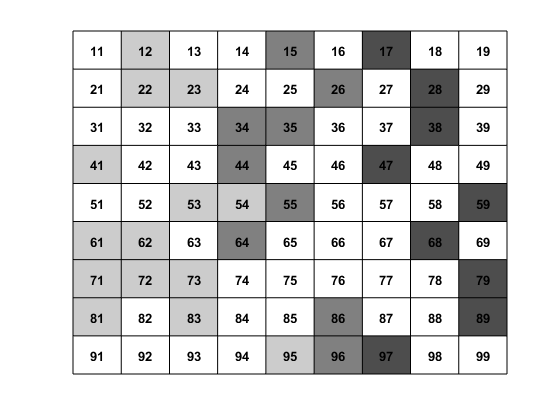}
\caption{GA algorithm \cite{deshkar2015solar}}
\label{fig_GA_array}
\end{subfigure}
\begin{subfigure}[b]{0.49\textwidth}
\centering
\includegraphics[width=2.5in]{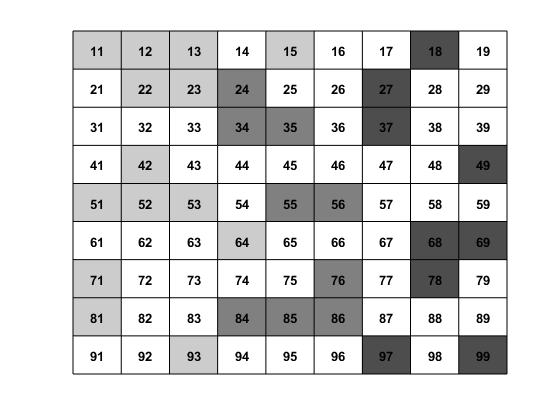}
\caption{PMSO algorithm}
\label{fig_PMSO_array}
\end{subfigure}
\caption{The different configurations of solar PV array.}
\label{fig_solar1}
\end{figure*}

\begin{figure}[!t]
\centering
\includegraphics[width=3in]{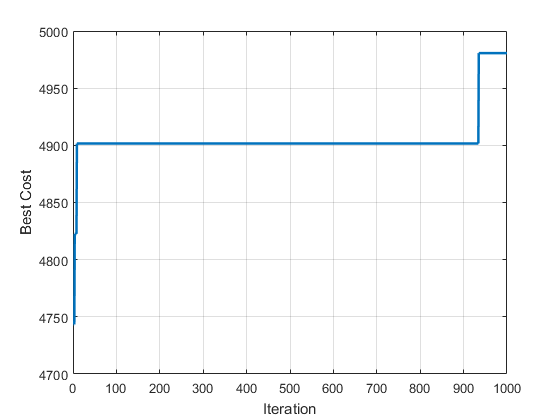}
\caption{Best solutions found by PMSO in different iterations.}
\label{fig_solar2}
\end{figure}

\begin{figure}[!t]
\centering
\includegraphics[width=3in]{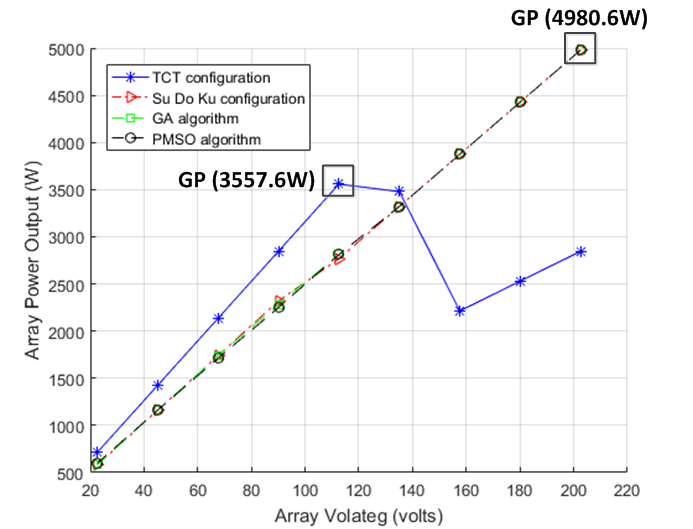}
\caption{Array power output against array voltage in different configurations.}
\label{fig_solar3}
\end{figure} 

\subsection{Experiments on the Solar PV Arrays}
Three different shadings on the solar PV array are experimented in this section, as it is performed in \cite{deshkar2015solar}. We compare our algorithm with TCT, Su Do Ku, and \cite{deshkar2015solar} which uses GA. According to tables and reported numbers of \cite{deshkar2015solar}, we concluded that they have chosen $V_m$ and $I_m$ to be 22.512 and 3.902, respectively. They are set similarly in this work.

The partial shading case which is tested here is the short and wide shadow. 
The different configurations of the solar PV array, in this case, are depicted in figure \ref{fig_solar1}. 
The shadow is shading on the array as shown in figure \ref{fig_TCT_array} which shows TCT configuration. 
The configurations of figures \ref{fig_SuDoKu_array} and \ref{fig_GA_array}, which are obtained respectively by Su Do Ku \cite{woyte2003partial} and GA \cite{deshkar2015solar} algorithms, are reported from \cite{deshkar2015solar}. The configuration of figure \ref{fig_PMSO_array} is obtained by the proposed PMSO algorithm. The curve of best solutions found by PMSO algorithm in the different iterations is depicted in figure \ref{fig_solar2}. The output power against the array voltage is also shown in figure \ref{fig_solar3}. As can be seen, the maximum power of both PMSO and GA algorithms are 4.9 KW while the maximum power of TCT configuration is 3.5 KW. This shows that the PMSO algorithm can perform well for this problem and finds the optimum solution which yields to the maximum possible output power. The calculations of solar array in different configurations are calculated according to the previous section and are listed in tables \ref{table_solar_calculations_1} and \ref{table_solar_calculations_2}.

%%%%%%%%%%%%%%%%%%%%%%%%%%%%%%%%%%%%%%%%%
\section{Conclusion}\label{Section_Conclusion}
In this paper, PMSO algorithm was introduced as a new metaheuristic optimization algorithm. One of the advantages of this algorithm in comparison to others is that it can escape from local bests better, in some cases. In this algorithm, exploration is satisfied by the sea wave whose strength is computed according to distance of particle from global best which is modeled as sea edge (coast). On the other hand, exploitation is performed by local search of each Gammarus. Details of PMSO algorithm are inspired by the maritime nature and behavior of Gammarus.
This algorithm, although is applicable on different types of optimization problems, was used in problem of partially shaded solar PV array and the experiments showed that this algorithm finds the configuration having the maximum output power.

% if have a single appendix:
%\appendix[Proof of the Zonklar Equations]
% or
%\appendix  % for no appendix heading
% do not use \section anymore after \appendix, only \section*
% is possibly needed

% use appendices with more than one appendix
% then use \section to start each appendix
% you must declare a \section before using any
% \subsection or using \label (\appendices by itself
% starts a section numbered zero.)
%

\appendices
%%%%%\section{Proof of the First Zonklar Equation}
%%%%%Appendix one text goes here.

% you can choose not to have a title for an appendix
% if you want by leaving the argument blank
%%%%%\section{}
%%%%%Appendix two text goes here.

% use section* for acknowledgment
%%%%%\section*{Acknowledgment}

%%%%%%%%%%%%%%%The authors would like to thank...

% Can use something like this to put references on a page
% by themselves when using endfloat and the captionsoff option.
\ifCLASSOPTIONcaptionsoff
  \newpage
\fi

% trigger a \newpage just before the given reference
% number - used to balance the columns on the last page
% adjust value as needed - may need to be readjusted if
% the document is modified later
%\IEEEtriggeratref{8}
% The "triggered" command can be changed if desired:
%\IEEEtriggercmd{\enlargethispage{-5in}}

% references section

% can use a bibliography generated by BibTeX as a .bbl file
% BibTeX documentation can be easily obtained at:
% http://mirror.ctan.org/biblio/bibtex/contrib/doc/
% The IEEEtran BibTeX style support page is at:
% http://www.michaelshell.org/tex/ieeetran/bibtex/
%\bibliographystyle{IEEEtran}
% argument is your BibTeX string definitions and bibliography database(s)
%\bibliography{IEEEabrv,../bib/paper}
%
% <OR> manually copy in the resultant .bbl file
% set second argument of \begin to the number of references
% (used to reserve space for the reference number labels box)

%\bibliographystyle{IEEEbib}
%\bibliographystyle{plain}
\bibliographystyle{IEEEtran}
\bibliography{References}

% Generated by IEEEtran.bst, version: 1.13 (2008/09/30)
\begin{thebibliography}{10}
\providecommand{\url}[1]{#1}
\csname url@samestyle\endcsname
\providecommand{\newblock}{\relax}
\providecommand{\bibinfo}[2]{#2}
\providecommand{\BIBentrySTDinterwordspacing}{\spaceskip=0pt\relax}
\providecommand{\BIBentryALTinterwordstretchfactor}{4}
\providecommand{\BIBentryALTinterwordspacing}{\spaceskip=\fontdimen2\font plus
\BIBentryALTinterwordstretchfactor\fontdimen3\font minus
  \fontdimen4\font\relax}
\providecommand{\BIBforeignlanguage}[2]{{%
\expandafter\ifx\csname l@#1\endcsname\relax
\typeout{** WARNING: IEEEtran.bst: No hyphenation pattern has been}%
\typeout{** loaded for the language `#1'. Using the pattern for}%
\typeout{** the default language instead.}%
\else
\language=\csname l@#1\endcsname
\fi
#2}}
\providecommand{\BIBdecl}{\relax}
\BIBdecl

\bibitem{talbi2009metaheuristics}
E.-G. Talbi, \emph{Metaheuristics: from design to implementation}.\hskip 1em
  plus 0.5em minus 0.4em\relax John Wiley \& Sons, 2009, vol.~74.

\bibitem{yang2010nature}
X.-S. Yang, \emph{Nature-inspired metaheuristic algorithms}.\hskip 1em plus
  0.5em minus 0.4em\relax Luniver press, 2010.

\bibitem{engelbrecht2007computational}
A.~P. Engelbrecht, \emph{Computational intelligence: an introduction}.\hskip
  1em plus 0.5em minus 0.4em\relax John Wiley \& Sons, 2007.

\bibitem{holland1975adaptation}
J.~Holland, ``Adaptation in natural and artificial systems, univ. of mich.
  press,'' \emph{Ann Arbor}, 1975.

\bibitem{goldberg1989genetic}
D.~E. Goldberg, ``Genetic algorithms in search, optimization, and machine
  learning, 1989,'' \emph{Reading: Addison-Wesley}, 1989.

\bibitem{koza1992genetic}
J.~R. Koza, \emph{Genetic programming: on the programming of computers by means
  of natural selection}.\hskip 1em plus 0.5em minus 0.4em\relax MIT press,
  1992, vol.~1.

\bibitem{fogel1966artificial}
L.~J. Fogel, A.~J. Owens, and M.~J. Walsh, ``Artificial intelligence through
  simulated evolution,'' 1966.

\bibitem{de1975analysis}
K.~A. De~Jong, ``Analysis of the behavior of a class of genetic adaptive
  systems,'' 1975.

\bibitem{koza1990genetic}
J.~R. Koza, \emph{Genetic programming: A paradigm for genetically breeding
  populations of computer programs to solve problems}.\hskip 1em plus 0.5em
  minus 0.4em\relax Stanford University, Department of Computer Science
  Stanford, CA, 1990.

\bibitem{glover1977heuristics}
F.~Glover, ``Heuristics for integer programming using surrogate constraints,''
  \emph{Decision Sciences}, vol.~8, no.~1, pp. 156--166, 1977.

\bibitem{kirkpatrick1983optimization}
S.~Kirkpatrick, C.~D. Gelatt, M.~P. Vecchi \emph{et~al.}, ``Optimization by
  simulated annealing,'' \emph{science}, vol. 220, no. 4598, pp. 671--680,
  1983.

\bibitem{kennedy2011particle}
J.~Kennedy, ``Particle swarm optimization,'' in \emph{Encyclopedia of machine
  learning}.\hskip 1em plus 0.5em minus 0.4em\relax Springer, 2011, pp.
  760--766.

\bibitem{gandomi2012krill}
A.~H. Gandomi and A.~H. Alavi, ``Krill herd: a new bio-inspired optimization
  algorithm,'' \emph{Communications in Nonlinear Science and Numerical
  Simulation}, vol.~17, no.~12, pp. 4831--4845, 2012.

\bibitem{guo2014new}
L.~Guo, G.-G. Wang, A.~H. Gandomi, A.~H. Alavi, and H.~Duan, ``A new improved
  krill herd algorithm for global numerical optimization,''
  \emph{Neurocomputing}, vol. 138, pp. 392--402, 2014.

\bibitem{saremi2014chaotic}
S.~Saremi, S.~M. Mirjalili, and S.~Mirjalili, ``Chaotic krill herd optimization
  algorithm,'' \emph{Procedia Technology}, vol.~12, pp. 180--185, 2014.

\bibitem{uymaz2015artificial}
S.~A. Uymaz, G.~Tezel, and E.~Yel, ``Artificial algae algorithm (aaa) for
  nonlinear global optimization,'' \emph{Applied Soft Computing}, vol.~31, pp.
  153--171, 2015.

\bibitem{woyte2003partial}
A.~Woyte, J.~Nijs, and R.~Belmans, ``Partial shadowing of photovoltaic arrays
  with different system configurations: literature review and field test
  results,'' \emph{Solar energy}, vol.~74, no.~3, pp. 217--233, 2003.

\bibitem{deshkar2015solar}
S.~N. Deshkar, S.~B. Dhale, J.~S. Mukherjee, T.~S. Babu, and N.~Rajasekar,
  ``Solar pv array reconfiguration under partial shading conditions for maximum
  power extraction using genetic algorithm,'' \emph{Renewable and Sustainable
  Energy Reviews}, vol.~43, pp. 102--110, 2015.

\bibitem{bidram2012control}
A.~Bidram, A.~Davoudi, and R.~S. Balog, ``Control and circuit techniques to
  mitigate partial shading effects in photovoltaic arrays,'' \emph{IEEE Journal
  of Photovoltaics}, vol.~2, no.~4, pp. 532--546, 2012.

\bibitem{rakesh2016performance}
N.~Rakesh and T.~V. Madhavaram, ``Performance enhancement of partially shaded
  solar pv array using novel shade dispersion technique,'' \emph{Frontiers in
  Energy}, vol.~10, no.~2, p. 227, 2016.

\bibitem{vernberg1983biology}
F.~Vernberg and W.~Vernberg, ``The biology of crustacea, vol. 8,'' 1983.

\bibitem{pennak1953fresh}
R.~W. Pennak, ``Fresh-water invertebrates of the united states,'' in
  \emph{Fresh-water invertebrates of the United States}.\hskip 1em plus 0.5em
  minus 0.4em\relax Ronald Press, 1953.

\bibitem{shamsaei2009effects}
M.~M. Shamsaei and S.~Khodami, ``(translation from persian:) the effects of
  different drying methods on quality and food factor of gammarus
  (pontogammarus maeoticus),'' \emph{Journal of fisheries}, vol.~3, no.~3,
  2009.

\bibitem{Abedian2003Analyzing}
A.~Abedian, M.~K. Khalesi, M.~Shokri, and M.~Heidari, ``(translation from
  persian:) analyzing reproduction and development of gammarus of caspian sea
  (pontogammarus maeoticus) in laboratory conditions,'' \emph{Journal of Iran
  Marine Science}, vol.~2, no. 2 \& 3, pp. 81--93, 2003.

\bibitem{Choubert1995}
G.~Choubert, J.-C.~G. Milicua, R.~Gomez, S.~Sanc{\'e}, H.~Petit,
  G.~N{\`e}gre-Sadargues, R.~Castillo, and J.-P. Trilles, ``Utilization of
  carotenoids from various sources by rainbow trout: muscle colour, carotenoid
  digestibility and retention,'' \emph{Aquaculture International}, vol.~3,
  no.~3, pp. 205--216, 1995.

\bibitem{ghareyazie2012studying}
B.~Ghareyazie and A.~Mottaghi, ``Studing pontogammarus maeoticus among southern
  coast of caspian sea,'' \emph{Middle-East Journal of Scientific Research},
  vol.~12, no.~11, pp. 1484--1487, 2012.

\bibitem{Seif2004Chemical}
S.~J. Seif-Abadi, H.~Negarestan, and B.~Moghadasi, ``(translation from
  persian:) chemical materials of body of pontogammarus maeoticus along the
  southern coast of caspian sea,'' \emph{Journal of Iran Marine Science},
  vol.~3, no.~1, pp. 51--55, 2004.

\bibitem{Azadkar2014Study}
Y.~Azadkar~Langroudi and N.~Shabanipour, ``(translation from persian:) study on
  gammarus species of the caspian sea (pontogammarus maeoticus) using sem
  images of mouthparts,'' \emph{Journal of Aquatic Physiology and
  Biotechnology}, vol.~1, no.~2, pp. 81--93, 2014.

\bibitem{blumenson1960derivation}
L.~Blumenson, ``A derivation of n-dimensional spherical coordinates,''
  \emph{The American Mathematical Monthly}, vol.~67, no.~1, pp. 63--66, 1960.

\bibitem{Shih2014ndimensional}
W.~Shih, ``n-dimension spherical coordinates and the volumes of the n-ball in
  $\mathbb{R}^n$,''
  \url{http://www.ams.sunysb.edu/~wshih/mathnotes/n-D_Spherical_coordinates.pdf},
  2014, [Online; Accessed October-2017].

\bibitem{suganthan2005problem}
P.~N. Suganthan, N.~Hansen, J.~J. Liang, K.~Deb, Y.-P. Chen, A.~Auger, and
  S.~Tiwari, ``Problem definitions and evaluation criteria for the cec 2005
  special session on real-parameter optimization,'' \emph{KanGAL report}, vol.
  2005005, pp. 1--50, 2005.

\bibitem{el2012performance}
M.~El-Abd, ``Performance assessment of foraging algorithms vs. evolutionary
  algorithms,'' \emph{Information Sciences}, vol. 182, no.~1, pp. 243--263,
  2012.

\bibitem{pham2011bees}
D.~Pham, A.~Ghanbarzadeh, E.~Koc, S.~Otri, S.~Rahim, and M.~Zaidi, ``The bees
  algorithm-a novel tool for complex optimisation,'' in \emph{Intelligent
  Production Machines and Systems-2nd I* PROMS Virtual International Conference
  (3-14 July 2006)}.\hskip 1em plus 0.5em minus 0.4em\relax sn, 2011.

\bibitem{karaboga2005idea}
D.~Karaboga, ``An idea based on honey bee swarm for numerical optimization,''
  Technical report-tr06, Erciyes university, engineering faculty, computer
  engineering department, Tech. Rep., 2005.

\bibitem{akay2009parameter}
B.~Akay and D.~Karaboga, ``Parameter tuning for the artificial bee colony
  algorithm.'' \emph{ICCCI}, vol. 2009, pp. 608--619, 2009.

\bibitem{mukhopadhyay2008population}
A.~Mukhopadhyay, A.~Roy, S.~Das, S.~Das, and A.~Abraham, ``Population-variance
  and explorative power of harmony search: an analysis,'' in \emph{Digital
  Information Management, 2008. ICDIM 2008. Third International Conference
  on}.\hskip 1em plus 0.5em minus 0.4em\relax IEEE, 2008, pp. 775--781.

\bibitem{socha2008ant}
K.~Socha and M.~Dorigo, ``Ant colony optimization for continuous domains,''
  \emph{European journal of operational research}, vol. 185, no.~3, pp.
  1155--1173, 2008.

\bibitem{storn1997differential}
R.~Storn and K.~Price, ``Differential evolution--a simple and efficient
  heuristic for global optimization over continuous spaces,'' \emph{Journal of
  global optimization}, vol.~11, no.~4, pp. 341--359, 1997.

\end{thebibliography}

% biography section
% 
% If you have an EPS/PDF photo (graphicx package needed) extra braces are
% needed around the contents of the optional argument to biography to prevent
% the LaTeX parser from getting confused when it sees the complicated
% \includegraphics command within an optional argument. (You could create
% your own custom macro containing the \includegraphics command to make things
% simpler here.)
%\begin{IEEEbiography}[{\includegraphics[width=1in,height=1.25in,clip,keepaspectratio]{mshell}}]{Michael Shell}
% or if you just want to reserve a space for a photo:

%\hfill \break

% if you will not have a photo at all:
%\begin{IEEEbiographynophoto}{John Doe}
%Biography text here.
%\end{IEEEbiographynophoto}

% insert where needed to balance the two columns on the last page with
% biographies
%\newpage

%\begin{IEEEbiographynophoto}{Jane Doe}
%Biography text here.
%\end{IEEEbiographynophoto}

% You can push biographies down or up by placing
% a \vfill before or after them. The appropriate
% use of \vfill depends on what kind of text is
% on the last page and whether or not the columns
% are being equalized.

%\vfill

% Can be used to pull up biographies so that the bottom of the last one
% is flush with the other column.
%\enlargethispage{-5in}

% that's all folks
\end{document}